\documentclass[sigconf,nonacm]{aamas}
\usepackage{balance}
\usepackage{natbib}
\usepackage[utf8]{inputenc}
\usepackage[T1]{fontenc}
\usepackage{url}
\usepackage{amsfonts}
\usepackage{nicefrac}
\usepackage{microtype}
\usepackage{xcolor}
\usepackage{afterpage}
\usepackage{cuted}
\usepackage{color, colortbl}
\usepackage{graphicx}
\usepackage{wrapfig,lipsum,booktabs}
\usepackage{bm}
\usepackage{tikz}
\usepackage{algorithm}
\usepackage[noend]{algpseudocode}
\usepackage{mdframed}
\usepackage{enumitem}
\usepackage{makecell}
\usepackage{float}
\usepackage{adjustbox}
\usepackage{multirow}
\usepackage{tabularx}
\usepackage{mathtools}
\usepackage[figuresright]{rotating}
\usepackage{pdflscape}
\usepackage{float}
\usepackage{cleveref}
\usepackage[english]{babel}
\usepackage{amsthm}

\newcommand{\commentsymbol}{//}
\algrenewcommand\algorithmiccomment[1]{\hfill \commentsymbol{} #1}
\makeatletter
\newcommand{\LineComment}[2][\algorithmicindent]{\Statex \hspace{#1}\commentsymbol{} #2}
\makeatother

\usepackage{breqn}

\usepackage[english]{babel}

\renewcommand{\arraystretch}{1.4}

\addto\extrasenglish{

}

\usepackage{etoolbox}
\newtoggle{inappendix}
\togglefalse{inappendix}

\apptocmd{\appendix}{\toggletrue{inappendix}}{}{\errmessage{failed to patch \appendix}}

\addto\extrasenglish{%
  
}

\makeatletter
\patchcmd{\hyper@makecurrent}{%
    \ifx\Hy@param\Hy@chapterstring
        \let\Hy@param\Hy@chapapp
    \fi
}{%
    \iftoggle{inappendix}{
        \@checkappendixparam{chapter}%
        \@checkappendixparam{section}%
        \@checkappendixparam{subsection}%
        \@checkappendixparam{subsubsection}%
    }{}%
}{}{\errmessage{failed to patch}}

\newcommand*{\@checkappendixparam}[1]{%
    \def\@checkappendixparamtmp{#1}%
    \ifx\Hy@param\@checkappendixparamtmp
        \let\Hy@param\Hy@appendixstring
    \fi
}
\makeatother

\DeclarePairedDelimiterX{\infdivx}[2]{(}{)}{%
  #1\;\delimsize\|\;#2%
}

\newcommand{\infdiv}{D_{\mathrm{KL}}\infdivx}

\newtheorem{theorem}{Theorem}
\newtheorem{corollary}{Corollary}
\newtheorem{lemma}{Lemma}

\definecolor{LightCyan}{rgb}{0.75,0.9,1}
\usepackage{caption} 
\newcommand{\paddedcolorbox}[2]{
  \begingroup
  \setlength{\fboxsep}{5pt}
  \colorbox{#1}{#2}
  \endgroup
}

\usepackage{paracol}
\usepackage{afterpage}
\usepackage{blindtext}

\makeatletter
\gdef\@copyrightpermission{
  \begin{minipage}{0.2\columnwidth}
   \href{https://creativecommons.org/licenses/by/4.0/}{\includegraphics[width=0.90\textwidth]{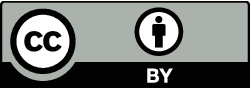}}
  \end{minipage}\hfill
  \begin{minipage}{0.8\columnwidth}
   \href{https://creativecommons.org/licenses/by/4.0/}{This work is licensed under a Creative Commons Attribution International 4.0 License.}
  \end{minipage}
  \vspace{5pt}
}
\makeatother

\setcopyright{ifaamas}
\acmConference[AAMAS '25]{Proc.\@ of the 24th International Conference
on Autonomous Agents and Multiagent Systems (AAMAS 2025)}{May 19 -- 23, 2025}
{Detroit, Michigan, USA}{Y.~Vorobeychik, S.~Das, A.~Nowé  (eds.)}
\copyrightyear{2025}
\acmYear{2025}
\acmDOI{}
\acmPrice{}
\acmISBN{}

\acmSubmissionID{6}

\title{Enhancing Offline Reinforcement Learning with Curriculum Learning-Based Trajectory Valuation}

\author{Amir Abolfazli}
\affiliation{
  \institution{L3S Research Center}
  \city{Hannover}
  \country{Germany}}
\email{abolfazli@l3s.de}

\author{Zekun Song}
\affiliation{
  \institution{Technical University of Berlin}
  \city{Berlin}
  \country{Germany}}
\email{zekun.song@tu-berlin.de}

\author{Avishek Anand}
\affiliation{
  \institution{Delft University of Technology}
  \city{Delft}
  \country{Netherlands}}
\email{avishek.anand@tudelft.nl}

\author{Wolfgang Nejdl}
\affiliation{
  \institution{L3S Research Center}
  \city{Hannover}
  \country{Germany}}
\email{nejdl@l3s.de}

\begin{abstract}
The success of deep reinforcement learning (DRL) relies on the availability and quality of training data, often requiring extensive interactions with specific environments. In many real-world scenarios, where data collection is costly and risky, offline reinforcement learning (RL) offers a solution by utilizing data collected by domain experts and searching for a batch-constrained optimal policy. This approach is further augmented by incorporating external data sources, expanding the range and diversity of data collection possibilities. However, existing offline RL methods often struggle with challenges posed by non-matching data from these external sources. In this work, we specifically address the problem of source-target domain mismatch in scenarios involving mixed datasets, characterized by a predominance of source data generated from random or suboptimal policies and a limited amount of target data generated from higher-quality policies. To tackle this problem, we introduce Transition Scoring (TS), a novel method that assigns scores to transitions based on their similarity to the target domain, and propose Curriculum Learning-Based Trajectory Valuation (CLTV), which effectively leverages these transition scores to identify and prioritize high-quality trajectories through a curriculum learning approach. Our extensive experiments across various offline RL methods and MuJoCo environments, complemented by rigorous theoretical analysis, demonstrate that CLTV enhances the overall performance and transferability of policies learned by offline RL algorithms.
\end{abstract}


\newcommand{\BibTeX}{\rm B\kern-.05em{\sc i\kern-.025em b}\kern-.08em\TeX}

\begin{document}

\pagestyle{fancy}
\fancyhead{}

\maketitle 


\section{Introduction}
Offline Reinforcement Learning (RL) is a class of RL methods that requires the agent to learn from a dataset of pre-collected experiences without further environment interaction~\citep{lange2012batch}. This learning paradigm decouples exploration from exploitation, rendering it particularly advantageous in scenarios where the process of data collection is costly, time-consuming, or risky~\citep{isele2018safe,fujimoto2019off}. 

By utilizing pre-collected datasets, offline RL can bypass the technical challenges that are associated with online data collection, and has potential benefits for a number of real environments, such as human-robot collaboration and autonomous systems~\citep{breazeal2008learning,tang2021model}.

However, this task is challenging, as offline RL methods suffer from the \emph{extrapolation error}~\citep{fujimoto2019off,kumar2019stabilizing}. This issue arises when offline deep RL methods are trained under one distribution but evaluated on a different one. More specifically, value functions implemented by a function approximator have a tendency to predict unrealistic values for unseen state-action pairs for standard off-policy deep RL algorithms such as BCQ~\citep{fujimoto2019off}, TD3+BC~\citep{fujimoto2021minimalist}, CQL~\citep{kumar2020conservative} and IQL~\citep{kostrikov2021offline}. This highlights the necessity for approaches that restrict the action space, forcing the agent to learn a behavior that is closely aligned with on-policy with respect to a subset of the source data~\citep{fujimoto2019off}.

In recent years, there have been a number of efforts within the paradigm of supervised learning for overcoming the \emph{source-target domain mismatch problem} valuating data, including \emph{data Shapley}~\citep{ghorbani2019data} and \emph{data valuation using reinforcement learning} (DVRL)~\citep{yoon2020data}. Such methods have shown promising results on several application scenarios such as robust learning and domain adaptation~\citep{yoon2020data}.

Despite the success of such methods in the supervised learning setting, adapting them to the offline reinforcement learning (RL) setting presents several challenges. One major issue is the non-i.i.d. nature of the data. In supervised learning, data samples are typically assumed to be independent and identically distributed (i.i.d.), but this assumption is violated in offline RL since the data is generated by an agent interacting with an environment~\citep{levine2020offline}. The presence of correlated and non-i.i.d. data samples complicates the valuation of these transitions and hinders generalization to the target domain. Additionally, the data distribution in offline RL often changes due to the agent's evolving policy or the environment's dynamics, leading to \emph{distributional shifts} that aggravate the valuation of transitions, as their value may fluctuate with changing dynamics~\citep{kumar2019stabilizing,kumar2020conservative}. 

Furthermore, unlike supervised learning, where the objective is to optimize a loss function given input-output pairs, the objective in RL is to maximize cumulative rewards by learning a policy that maps states to actions~\citep{sutton2018reinforcement}. Therefore, designing a reward function that accounts for distributional similarities between transition items from different domains is essential for effective policy learning.

These challenges highlight the complexities involved in adapting supervised learning methods to the offline RL setting and underscore the need for novel approaches to handle these challenges.

However, a recent work~\citep{cai2023curriculum} introduced CUORL, a curriculum learning-based approach aimed at enhancing the performance of offline RL methods by strategically selecting valuable transition items. The limitation of CUORL is that only the current policy is considered to valuate trajectories, which lacks consideration of the information in the target dataset, making it challenging for the agent to adapt to different environments. 

Another recent work~\citep{hong2023harnessing} introduced Harness, which addresses a major challenge in offline RL involving mixed datasets, where the prevalence of low-return trajectories can limit the effectiveness of advanced algorithms, preventing them from fully exploiting high-return trajectories. Harness tackles this problem by re-weighting the dataset sampling process to create an artificial dataset that results in a behavior policy with a higher overall return, enabling RL algorithms to better utilize high-performing trajectories. By optimizing the sampling strategy, Harness enhances the performance of offline RL algorithms in environments characterized by mixed-return datasets. However, by re-weighting the dataset sampling process, Harness introduces a bias in favor of high-return trajectories, potentially neglecting important and high-quality transition items from low-return trajectories that could enhance the robustness of the policy.

In this work, we introduce a transition scoring (TS) method, which assigns a score to each transition in a trajectory, and a curriculum learning-based trajectory valuation method (CLTV) that operates based on the scores computed by TS to enable the agent to identify and select the most promising trajectories. Unlike existing methods, our approach leverages high-quality transitions from different trajectories generated by various policies. 

Our results, using existing methods (CUORL and Harness) and two base offline RL algorithms (CQL and IQL) in four MuJoCo environments, show that our CLTV method improves the performance and transferability of offline RL policies. We also provide a theoretical analysis to demonstrate the efficacy of our approach.

\section{Background}
\label{sec:background}
\noindent\textbf{Reinforcement Learning.}
The RL problem is typically modeled by a Markov decision process (MDP), formulated as a tuple \scalebox{0.9}{$(\mathcal{X}, \mathcal{U}, p, r, \gamma)$}, with a state space $\mathcal{X}$, an action space $\mathcal{U}$, and transition dynamics $p$.

At each discrete time step, the agent performs an action \( u \in \mathcal{U} \) in a state \( x \in \mathcal{X} \), transitions to a new state \( x^{\prime} \in \mathcal{X} \) based on the transition dynamics \(p(x^{\prime} \mid x, u)\), and receives a reward \( r(x, u, x^{\prime}) \). The agent's goal is to maximize the expectation of the sum of discounted rewards, also known as the return \( R_{t} = \sum_{i=t+1}^{\infty} \gamma^{i} r(x_{i}, u_{i}, x_{i+1}) \), which weights future rewards with respect to the discount factor \( \gamma \in [0,1) \), determining the effective horizon. The agent makes decisions based on the policy \( \pi: \mathcal{X} \rightarrow \mathcal{P}(\mathcal{U}) \), which maps a given state \( x \) to a probability distribution over the action space \( \mathcal{U} \). For a given policy \( \pi \), the value function is defined as the expected return of an agent starting from state \( x \), performing action \( u \), and following the policy \( Q^{\pi}(x, u) = \mathbb{E}_{\pi}[R_{t} \mid x, u] \). The state-action value function can be computed through the Bellman equation for the Q function:
\begin{equation}
Q^{\pi}(x, u)=\mathbb{E}_{s^{\prime} \sim p}\left[r\left(x, u, x^{\prime}\right)+\gamma \mathbb{E}_{u^{\prime} \sim \pi} Q^{\pi}(x^{\prime}, u^{\prime})\right].
\end{equation}

Given $Q^{\pi}$, the optimal policy $\pi^{*}=\operatorname{max}_{u} Q^{*}(x, u)$, can be obtained by greedy selection over the optimal value function $Q^{*}(x, u)=\max_{\pi} Q^{\pi}(x, u)$. 

For environments confronting agents with the curse of dimensionality, the value can be estimated with a
differentiable function approximator $Q_\theta(x, u)$, with parameters $\theta$.

\noindent\textbf{Offline Reinforcement Learning.}
Standard off-policy deep RL algorithms such as deep Q-learning (DQN)~\citep{mnih2015human} and deep deterministic policy gradient (DDPG)~\citep{lillicrap2015continuous} are applicable in batch RL as they are based on more fundamental batch RL algorithms~\citep{fujimoto2019benchmarking}. 

However, they suffer from a phenomenon, known as \textit{extrapolation error}, which occurs when there is a mismatch between the given fixed batch of data and true state-action visitation of the current policy~\citep{fujimoto2019off}. This is problematic as incorrect values of state-action pairs, not contained in the batch, are propagated through temporal difference updates of most off-policy algorithms~\citep{sutton1988learning}, resulting in poor performance of the model~\citep{thrun1993issues}. 

In offline RL, the goal is to learn a policy based on a previously collected dataset of transition items, optimizing decision-making without any additional interaction with the environment.
\section{Problem Description} 
\label{sec:problem_desc}
We assume the availability of a source dataset \scalebox{0.8}{{$\mathcal{D_S}= \{(x_{i}^{\mathcal{S}}, u^{\mathcal{S}}_{i}, {x^{\prime}_{i}}^{\mathcal{S}},$}} \scalebox{0.8}{{$ {r}_{i}^{\mathcal{S}})\}_{i=1}^{N} \sim \mathcal{P}_{\mathcal{S}}$}} and a target dataset \scalebox{0.8}{$\mathcal{D_T}=$ $\{(x_{i}^{\mathcal{T}}, u^{\mathcal{T}}_{i}, {x^{\prime}_{i}}^{\mathcal{T}}, {r}_{i}^{\mathcal{T}})\}_{i=1}^{M} \sim \mathcal{P}_{\mathcal{T}}$}, where $x \in \mathbb{R}^{m}$ is a state; $u \in \mathbb{R}^{n}$ is the action that the agent performs at the state $x$; $r \in \mathbb{R}$ is the reward that the agent gets by performing the action $u$ in the state $x$; and $x^{\prime} \in \mathbb{R}^{m}$ is the state that the agent transitions to (i.e., next state). We also assume that the target dataset $\mathcal{D_T}$ is much smaller than the the source dataset $\mathcal{D_S}$, therefore $N \gg M$. Furthermore, the source distribution $\mathcal{P}_{\mathcal{S}}$ can be different from the target distribution $\mathcal{P}_{\mathcal{T}}$ (i.e. $\mathcal{P}_{\mathcal{S}} \neq \mathcal{P}_{\mathcal{T}}$), confronting the agent with the source-target domain mismatch problem.

Moreover, while intrinsic reward functions may be the same across source and target domains (i.e., $r^{\mathcal{S}} = r^{\mathcal{T}}$), modifications to the reward functions in either domain can lead to effective differences (i.e., $r^{\mathcal{S}} \neq r^{\mathcal{T}}$). We assume that the reward functions can potentially be different across the domains. These potential differences in reward functions could exacerbate the domain mismatch and introduce additional challenges for effective knowledge transfer. Each trajectory $\tau_k^{\mathcal{S}}$ in $\mathcal{D_S}$ is defined as a sequence of transitions \scalebox{0.8}{$\tau_k^{\mathcal{S}} = \{(x_{i}^{k,\mathcal{S}}, u_{i}^{k,\mathcal{S}}, {x^{\prime}_{i}}^{k,\mathcal{S}}, r_{i}^{k,\mathcal{S}})\}_{i=1}^{L_k^{\mathcal{S}}}$}, where $L_k^{\mathcal{S}}$ denotes the length of $k$-th trajectory. The degree of similarity between these transitions and those in the target dataset $\mathcal{D_T}$ varies.

The goal is to identify high-quality trajectories by quantifying the similarity of each trajectory $\tau_k^{\mathcal{S}}$ from the source dataset to the target dataset $\mathcal{D_T}$, aggregating the similarities of the individual transitions within $\tau_k^{\mathcal{S}}$ to those in $\mathcal{D_T}$, and selecting trajectories most relevant for effective knowledge transfer.

\section{Related Work}
\label{sec:related_work}
\noindent\textbf{Offline RL.}
The RL literature contains numerous techniques for dealing with the source-target domain mismatch problem. DARLA ~\citep{higgins2017darla}, a zero-shot transfer learning method that learns disentangled representations that are robust against domain shifts; and SAFER~\citep{slack2022safer}, which accelerates policy learning on complex control tasks by considering safety constraints. Meanwhile, the literature on off-policy RL includes principled experience replay memory sampling techniques. Prioritized Experience Replay (PER)~\citep{schaul2015prioritized} (e.g., \citep{hou2017novel,horgan2018distributed,kang2021deep}) attempts to sample transitions that contribute the most toward learning. However, most of the work to date on offline RL is focused on preventing the training policy from being too disjoint with the behavior policy~\citep{fujimoto2019off,kumar2020conservative,kidambi2020morel}. 

Batch-Constrained deep Q-learning (BCQ)~\citep{fujimoto2019off} is an offline RL method for continuous control, restricting the action space, thereby eliminating actions that are unlikely to be selected by the behavior policy and therefore rarely observed in the batch. 

Conservative Q-Learning (CQL)~\citep{kumar2020conservative} prevents the training policy from overestimating the Q-values by utilizing a penalized empirical RL objective. More precisely, CQL optimizes the value function not only to minimize the temporal difference error based on the interactions seen in the dataset but also minimizes the value of actions that the currently trained policy takes, while at the same time maximizing the value of actions taken by the behavior policy during data generation.

Twin Delayed Deep Deterministic (TD3) policy gradient with Behavior Cloning (BC) is a model-free algorithm that trains a policy to emulate the behavior policy from the data~\citep{fujimoto2021minimalist}. TD3-BC is an adaptation of TD3~\citep{fujimoto2018addressing} to the offline setting, adding a behavior cloning term to policy updates to encourage agents to align their actions with those found in the dataset. 

Advantage Weighted Actor-Critic (AWAC)~\citep{nair2020awac} uses a Q-function to estimate the advantage to increase sample efficiency.

To increase the generalization capability of offline RL methods, ~\citep{kostrikov2021offline} propose in-sample Q-learning (IQL), approximating the policy improvement step by considering the state value function as a random variable with some randomness determined by the action, and then taking a state-conditional expectile of this random variable to estimate the value of the best actions in that state. 
Policy in Latent Action Space (PLAS)~\citep{zhou2021plas} constrains actions to be within the support of the behavior policy. PLAS disentangles the in-distribution and out-of-distribution generalization of actions, enabling fine-grained control over the method's generalization. \citep{prudencio2023survey} provide a comprehensive overview of offline RL. 
~\citep{hong2023harnessing} introduce Harness, which addresses a challenge in offline RL involving mixed datasets. The prevalence of low-return trajectories in these datasets can limit the effectiveness of algorithms, preventing them from fully exploiting high-return trajectories. Harness tackles this problem by re-weighting the dataset sampling process to create an artificial dataset that results in a behavior policy with a higher return. This enables RL algorithms to better utilize high-performing trajectories.

\vspace{0.5em}
\noindent\textbf{Curriculum Learning in RL.}
In reinforcement learning, curriculum learning (CL) sequences tasks in a strategic manner, enhancing agent performance and accelerating training, especially for complex challenges~\citep{narvekar2020curriculum,turchetta2020safe,liu2021curriculum,ren2018self}. CL has shown promising performance in real-world applications~\citep{xu2019macro,el2020student,matavalam2022curriculum}, helping agents solve complex problems and transfer knowledge across different tasks~\citep{narvekar2020curriculum,klink2021boosted}.

Among the works on curriculum learning and transition valuation in RL, the most relevant to ours is CUORL~\citep{cai2023curriculum}. CUORL enhances offline RL methods by training agents on pertinent trajectories but struggles with non-deterministic policy properties. Additionally, CUORL only considers the current policy to evaluate trajectories, neglecting information from the target dataset, which hinders the agent's adaptability to different environments. 

Our method addresses these issues by using KL divergence to capture non-deterministic properties and by identifying and selecting high-quality transition items from different trajectories.
\section{Methods}
\label{sec:proposed_method}
Our proposed method includes two key components: Transition Scoring (TS) and Curriculum Learning-Based Trajectory Valuation (CLTV). TS assigns scores to transitions based on their relevance to the target domain, helping to identify high-quality transitions. 

CLTV leverages TS scores to prioritize high-quality trajectories through curriculum learning, guiding the agent's development in a structured manner. These components are discussed below.

\noindent\textbf{Transition Scoring (TS).}
Inspired by DVRL~\citep{yoon2020data}, we adopt the REINFORCE algorithm~\citep{williams1992simple} and use a DNN $v_{\phi}$ as the transition scoring method. The goal is to find the parameters $\phi^*$ of the DNN  so that the network returns the optimal probability distribution over the set of all possible transitions.

The TS model $v_{\phi}: \mathcal{X} \times \mathcal{U} \times \mathcal{X^{\prime}} \rightarrow[0,1]$ is optimized to output scores corresponding to the similarity of transitions in the source dataset to the target domain. We formulate the corresponding optimization problem as:

\begin{equation}
\resizebox{0.72\hsize}{!}{%
$\max _\phi J\left(\pi_\phi\right)=\mathbb{E}_{\left(x^{\mathcal{S}}, u^{\mathcal{S}}, x^{\prime \mathcal{S}}\right) \sim P^{\mathcal{S}}}\left[r_\phi\left(x^{\mathcal{S}}, u^{\mathcal{S}}, x^{\prime \mathcal{S}}, \Delta_\theta\right)\right]$,
}
\end{equation}
where
\begin{equation}
\label{eq-reward}
\resizebox{0.72\hsize}{!}{
$r_\phi\left(x^{\mathcal{S}}, u^{\mathcal{S}}, x^{\prime \mathcal{S}}, \Delta_\theta\right)=2 \cdot \frac{r_\phi^{\prime}\left(x^{\mathcal{S}}, u^{\mathcal{S}}, x^{\prime \mathcal{S}}, \Delta_\theta\right)-r_{\text {min }}^{\prime}}{r_{\text {max }}^{\prime}-r_{\text {min }}^{\prime}}-1$.
}
\end{equation}

\autoref{eq-reward} represents the normalized reward, keeping the rewards in the range \([-1, 1]\), where \scalebox{0.8}{$r_\phi^{\prime}\left(x^{\mathcal{S}}, u^{\mathcal{S}}, x^{\prime \mathcal{S}}, \Delta_\theta\right)$} is the unnormalized reward and defined as follows:

\begin{equation}
\label{eqq4}
\resizebox{0.85\hsize}{!}{%
\begin{math}
\begin{split}
r_\phi^{\prime}\left(x^{\mathcal{S}}, u^{\mathcal{S}}, x^{\prime \mathcal{S}}, \Delta_\theta\right) &= \delta \cdot \sum_{\left(x^{\mathcal{S}}, u^{\mathcal{S}}, x^{\prime \mathcal{S}}\right) \in \mathcal{B}} v_\phi\left(x^{\mathcal{S}}, u^{\mathcal{S}}, x^{\prime \mathcal{S}}\right) \cdot \Delta_\theta\left(x^{\mathcal{S}}, u^{\mathcal{S}}, x^{\prime \mathcal{S}}\right) \\
&\quad + (1-\delta) \cdot \sum_{\left(x^{\mathcal{S}}, u^{\mathcal{S}}, x^{\prime \mathcal{S}}\right) \in \mathcal{B}} v_\phi\left(x^{\mathcal{S}}, u^{\mathcal{S}}, x^{\prime \mathcal{S}}\right).
\end{split}
\end{math}
}
\end{equation}

In \autoref{eqq4}, the term \scalebox{0.8}{$\sum_{\left(x^{\mathcal{S}}, u^{\mathcal{S}}, x^{\prime \mathcal{S}}\right) \in \mathcal{B}} v_\phi\left(x^{\mathcal{S}}, u^{\mathcal{S}}, x^{\prime \mathcal{S}}\right)$} regularizes the unnormalized reward and prevents assigning low scores to the majority of transitions. \(\Delta_\theta\) represents the dynamics factor, quantified by the classifiers \(q_{xu} = p_{\theta_{xu}}(y \mid x, u)\) and \(q_{xu{x}^{\prime}} = p_{\theta_{xu{x}^{\prime}}}(y \mid x, u, {x}^{\prime})\), where \scalebox{1}{\(y \in \{\text{source}, \text{target}\}\)}.

The optimization problem involves adjusting \(\phi\) to maximize the expected reward by aligning transition dynamics with the target domain. The classifiers, parameterized by \(\theta\), remain fixed during the optimization of \(\phi\). These classifiers are implemented as fully connected multi-layer neural networks, similar to those used in DARC~\citep{eysenbach2020off}. They measure the likelihood of transitions belonging to either the source or target domains. By applying Bayes' rule, the classifier probabilities are related to transition probabilities, enabling the expression of posterior probabilities in terms of likelihoods and prior probabilities. 

The classifier \(q_{xu}\) is defined as follows:
\begin{equation}
\label{qxu}
\resizebox{0.55\hsize}{!}{%
\begin{math}
q_{xu} = p_{\theta_{x u}}\left(y \mid x^{\mathcal{S}}, u^{\mathcal{S}}\right)=\frac{p\left(x^{\mathcal{S}}, u^{\mathcal{S}} \mid y\right) p(y)}{p\left(x^{\mathcal{S}}, u^{\mathcal{S}}\right)}.
\end{math}
}
\end{equation}

Similarly, the classifier $q_{xu{x}^{\prime}}$ is defined as follows:
\begin{equation}
\label{qxux}
\resizebox{0.72\hsize}{!}{%
\begin{math}
q_{xu{x}^{\prime}}=p_{\theta_{x u x^{\prime}}}\left(y \mid x^{\mathcal{S}}, u^{\mathcal{S}}, x^{\prime \mathcal{S}}\right)=\frac{p\left(x^{\mathcal{S}}, u^{\mathcal{S}}, x^{\prime \mathcal{S}} \mid y\right) p(y)}{p\left(x^{\mathcal{S}}, u^{\mathcal{S}}, x^{\prime \mathcal{S}}\right)}.
\end{math}
}
\end{equation}

\newpage
We train the classifiers $q_{xu{x}^{\prime}}$ and $q_{xu}$ to minimize the standard cross-entropy loss. The loss function for the classifier $q_{xu}$ is presented in \autoref{loss-xu}.

\begin{equation}
\label{loss-xu}
\resizebox{0.72\hsize}{!}{
\begin{math}
\begin{aligned}
\mathcal{L}_{xu}\left(\theta_{xu}\right) = & -\mathbb{E}_{\mathcal{D}_{\text{target}}}\left[\log q_{\theta_{xu}}\left(\text{target} \mid x^{\mathcal{S}}, u^{\mathcal{S}}\right)\right] \\
& -\mathbb{E}_{\mathcal{D}_{\text{source}}}\left[\log q_{\theta_{xu}}\left(\text {source} \mid x^{\mathcal{S}}, u^{\mathcal{S}}\right)\right] .
\end{aligned}
\end{math}
}
\end{equation}

Similarly, the loss function for the classifier $q_{xu{x}^{\prime}}$ is presented in \autoref{loss-xux}.

\begin{equation}
\label{loss-xux}
\resizebox{0.85\hsize}{!}{
\begin{math}
\begin{aligned}
\mathcal{L}_{xu{x}^{\prime}}\left(\theta_{xu{x}^{\prime}}\right) = & -\mathbb{E}_{\mathcal{D}_{\text {target}}}\left[\log q_{\theta_{xu{x}^{\prime}}}\left(\text{target} \mid x^{\mathcal{S}}, u^{\mathcal{S}}, x^{\prime \mathcal{S}}\right)\right] \\
& -\mathbb{E}_{\mathcal{D}_{\text{source}}}\left[\log q_{\theta_{xu{x}^{\prime}}}\left(\text {source} \mid x^{\mathcal{S}}, u^{\mathcal{S}}, x^{\prime \mathcal{S}}\right)\right]
\end{aligned}
\end{math}
}
\end{equation}

The dynamics factor $\Delta$ (as defined in ~\citep{eysenbach2020off}) measures the discrepancy between the distribution of the source and target datasets and is defined as follows:

\begin{equation}
\label{eq-delata}
\resizebox{0.70\hsize}{!}{%
\begin{math}
\begin{aligned}
\Delta_\theta\left(x^{\mathcal{S}}, u^{\mathcal{S}}, x^{\prime \mathcal{S}}\right)= & \log q_{\theta_{x u x^{\prime}}}\left(\operatorname{target} \mid x^{\mathcal{S}}, u^{\mathcal{S}}, x^{\prime \mathcal{S}}\right) \\
& -\log q_{\theta_{x u}}\left(\operatorname{target} \mid x^{\mathcal{S}}, u^{\mathcal{S}}\right) \\
& -\log q_{\theta_{x u x^{\prime}}}\left(\operatorname{source} \mid x^{\mathcal{S}}, u^{\mathcal{S}}, x^{\prime \mathcal{S}}\right) \\
& +\log q_{\theta_{x u}}\left(\operatorname{source} \mid x^{\mathcal{S}}, u^{\mathcal{S}}\right).
\end{aligned}
\end{math}
}
\end{equation}

\color{black}

Optimizing $\Delta$ enables the identification of transitions with lower distribution discrepancies. $w$ is the sum of selection probabilities, reflecting the likelihood of each transition to be selected. This factor prevents the TS model from selecting only a limited number of transitions. The parameters \(\delta\) and \(1-\delta\) serve as balancing factors for $v_{\phi}  \Delta$ and \(v_{\phi}\), respectively, in the calculation of \(r^{\prime}_\phi\).

For training the TS, all transitions in the source buffer \(\mathcal{D}_{\mathcal{S}}\) are divided into batches. Each batch \(D_B^{\mathcal{S}} = \left(t_j\right)_{j=1}^{B_s} \sim \mathcal{D}^{\mathcal{S}}\) is provided as input to the TS, with shared parameters across the batch. Let \(w_j = v_\phi\left(x_j^{\mathcal{S}}, u_j^{\mathcal{S}}, x_j^{\prime \mathcal{S}}\right)\) denote the probability that transition item \(j\) from the source buffer is selected for training the offline RL model.

The sampling step \(z_j \sim \text{Bernoulli}(w_j)\) is then used to choose transition items based on their importance scores \(w_j\), following the approach in \citep{yoon2020data}. Here, $z_j$ is a binary indicator that decides whether the corresponding transition is selected ($z_j=1$) or not ($z_j=0$). This method prioritizes high-quality transitions by sampling them with higher probability, effectively focusing on those that contribute most to the performance of the RL agent.

Our adopted version of the REINFORCE algorithm has the following objective function for the policy $\pi_\phi$:

\begin{equation}
\label{objective}
\resizebox{0.75\hsize}{!}{%
\begin{math}
\begin{aligned}
J\left(\pi_{\phi}\right)&= \mathbb{E}_{\substack{\substack{(x^{\mathcal{S}}, u^{\mathcal{S}}, x^{\prime \mathcal{S}}) \sim P^{\mathcal{S}}\\ w \sim \pi_{\phi}(\mathcal{D_S}, \cdot)}}}\left[r_\phi\left(x^{\mathcal{S}}, u^{\mathcal{S}}, x^{\prime \mathcal{S}}, \Delta_\theta\right)\right] \\
&=\int P^{\mathcal{S}}\left((x^{\mathcal{S}}, u^{\mathcal{S}}, x^{\prime \mathcal{S}})\right) \sum_{w \in[0,1]^{N}} \pi_{\phi}(\mathcal{D_S}, w) \\
&\quad \cdot\left[r_\phi\left(x^{\mathcal{S}}, u^{\mathcal{S}}, x^{\prime \mathcal{S}}, \Delta_\theta\right)\right] d \left((x^{\mathcal{S}}, u^{\mathcal{S}}, x^{\prime \mathcal{S}})\right).
\end{aligned}
\end{math}
}
\end{equation}

In the above equation, \(\pi_{\phi}(\mathcal{D_S}, w)\) represents the probability of the selection probability vector \(w\) occurring. The policy uses the scores output by the TS. This contrasts with DVRL~\citep{yoon2020data}, which employs a binary selection vector \(\mathbf{s} = \left(s_1, \ldots, s_{B_s}\right)\), where \(s_{B_s}\) denotes the batch size, \(s_i \in \{0,1\}\), and \(P\left(s_i = 1\right) = w_i\). Thus, in our training, the TS does not control exploration, but instead provides scores for the transition items and is tuned accordingly.

We calculate the gradient of the above objective function (\autoref{objective}) as follows:

\begin{equation}
\label{grad-eq}
\resizebox{0.92\hsize}{!}{%
\begin{math}
\begin{aligned}
    \nabla_{\phi} J\left(\pi_{\phi}\right) 
    &= \mathbb{E}_{\substack{(x^{\mathcal{S}}, u^{\mathcal{S}}, x^{\prime \mathcal{S}}) \sim P^{\mathcal{S}}\\ w \sim \pi_{\phi}(\mathcal{D_S}, \cdot)}} \Bigg[ r_\phi\left(x^{\mathcal{S}}, u^{\mathcal{S}}, x^{\prime \mathcal{S}}, \Delta_\theta\right) \cdot \nabla_{\phi} \log \pi_{\phi}(\mathcal{D_S}, w) \\
    & \quad + \nabla_{\phi} r_\phi\left(x^{\mathcal{S}}, u^{\mathcal{S}}, x^{\prime \mathcal{S}}, \Delta_\theta\right) \Bigg].
\end{aligned}
\end{math}
}
\end{equation}

Then, the parameters $\phi$ of $J\left(\pi_{\phi}\right)$ are updated using the calculated gradient $\nabla_{\phi} J\left(\pi_{\phi}\right)$ as follows:
\begin{equation}
\label{eq:phi_update}
\phi \leftarrow \phi + k \left(r_{\phi} \cdot \nabla_{\phi} \log \pi_{\phi}(D_B^{\mathcal{S}}, (z_1, \ldots, z_{B_{\mathcal{S}}})) + \nabla_{\phi} r_\phi \right),
\end{equation}

where $k$ is the learning rate and $r_{\phi}$ denotes the normalized reward. 

The complete derivation of \autoref{grad-eq} is provided in \autoref{derivation}. Our transition scoring approach is presented in \autoref{alg:ts}.

\begin{algorithm}[!ht]
\caption{Transition Scoring (TS)}
\label{alg:ts}
\begin{algorithmic}[1]
\State \textbf{Input:} Source dataset $\mathcal{D_S}$, target dataset $\mathcal{D_T}$; \\classifiers $q_{xu}$  and $q_{xu{x}^{\prime}}$; \\ratio of transition similarity to reward $\delta$; \\learning rate $k$
\State \textbf{Output:} TS model $v_{\phi}$
\State Train TS $v_{\phi}(t_j)$ with transition item $t_j = (x_j, u_j, {x_{j}}^{\prime})$
\While{not converged}
    \State Sample a mini-batch: $D_B^{\mathcal{S}}=(t_j)_{j=1}^{B_{\mathcal{S}}} \sim \mathcal{D_S}$
    \State Initialize weighted sum $\mathcal{V}_{\phi} = 0$
    \State Initialize selection probability sum $\mathcal{W}_{\phi} = 0$
    \For{$j=1, \ldots, B_{\mathcal{S}}$}
        \State $w_j = v_{\phi}(t_j)$
        \State $z_j \sim \text{Bernoulli}(w_j)$
        \State $\Delta_j = \log \frac{q_{xu{x}^{\prime}}(\text{target} \mid x_j,u_j,{x_{j}}^{\prime})}{q_{xu{x}^{\prime}}(\text{source} \mid x_j,u_j,{x_{j}}^{\prime})} - \log \frac{q_{xu}(\text{target} \mid x_j,u_j)}{q_{xu}(\text{source} \mid x_j,u_j)}$
        \State $\mathcal{V}_{\phi} \leftarrow \mathcal{V}_{\phi} + w_j \cdot \Delta_j$
        \State $\mathcal{W}_{\phi} \leftarrow \mathcal{W}_{\phi} + w_j$
    \EndFor
    \State $r'_{\phi} = \delta \cdot \mathcal{V}_{\phi} + (1 - \delta) \cdot \mathcal{W}_{\phi}$
    \State $r_{\phi} = 2 \cdot \frac{r'_{\phi} - r'_{\min}}{r'_{\max} - r'_{\min}} - 1$
    \State $\phi \leftarrow \phi + k \cdot r_{\phi} \cdot \nabla_{\phi} \log \pi_{\phi}(D_B^{\mathcal{S}}, (z_1, \ldots, z_{B_{\mathcal{S}}}))$
\EndWhile
\end{algorithmic}
\end{algorithm}

\noindent\textbf{Curriculum Learning-based Trajectory Valuation (CLTV).} 
Building upon the transition scoring (TS) mechanism, CLTV focuses on enhancing policy learning by prioritizing valuable trajectories. CLTV employs curriculum learning to present the agent with the most valuable trajectories at different stages of learning. 

CLTV uses two distinct sets of actor-critic networks: one for the source domain \(\left(Q_{\omega_1}^{\mathcal{S}}, \pi_{\theta_1}^{\mathcal{S}}\right)\) and another for the target domain \(\left(Q_{\omega_2}^{\mathcal{T}}, \pi_{\theta_2}^{\mathcal{T}}\right)\). By maintaining separate networks, CLTV tailors the learning process to the unique dynamics and reward structures of each domain, facilitating effective policy transfer. 

CLTV integrates the TS method by training classifiers \(q_{xu{x}^{\prime}}\) and \(q_{xu}\) to adjust the rewards, helping to address the distributional shift between the source and target domains in offline RL. 

Each trajectory $\tau_k^{\mathcal{S}}$ in the source dataset $\mathcal{D}_{\mathcal{S}}$ is defined as a sequence of transitions $\tau_k^{\mathcal{S}}=\left(x_i^{k, \mathcal{S}}, u_i^{k, \mathcal{S}}, x_i^{\prime, k, \mathcal{S}}, r_i^{k, \mathcal{S}}\right)_{i=1}^{L_k^{\mathcal{S}}}$, where $L_k^{\mathcal{S}}$ denotes the length of the $k$-th trajectory. 

The degree of similarity between these transitions in $\mathcal{D}_{\mathcal{S}}$ and those in the target dataset $\mathcal{D}_{\mathcal{T}}$ varies. When the source data does not perfectly align with the target environment, relying solely on the original rewards can cause the agent to learn policies that are suboptimal or misaligned with the target setting. 

Moreover, not all transitions within a trajectory contribute equally to learning; only a subset of transitions contributes to improving the policy. By assigning a relevance score to each transition based on its similarity to the target domain dynamics, we effectively re-weight the rewards associated with these transitions. 

This approach increases the influence of transitions that are more representative of the target environment, guiding the agent to focus on the most valuable experiences. As a result, the agent learns policies that are better aligned with the target domain, enhancing overall performance. This method is especially beneficial in scenarios with sparse or noisy reward signals, as it helps the model identify and leverage valuable experiences.

During the curriculum learning process, the value of the trajectory \(i\) in the source dataset is computed as follows:

\begin{equation}
\resizebox{0.8\hsize}{!}{%
\begin{math}
v_i = \exp\left(-\infdiv{\pi_{\theta_2}^{\mathcal{T}}(u_t|x_t)}{\pi_{\theta_1}^{\mathcal{S}}(u_t|x_t)}\right) \cdot \sum_{t=1}^{T} \gamma^{t-1} r_t.
\end{math}
}
\end{equation}

This valuation formula uses the KL divergence between the target policy \(\pi_{\theta_2}^{\mathcal{T}}\) and the policy we need to train \(\pi_{\theta_1}^{\mathcal{S}}\), along with discounted rewards, to valuate trajectories. It includes two main components: the \emph{similarity component} \scalebox{0.9}{\(\exp\left(-\infdiv{\pi_{\theta_2}^{\mathcal{T}}(u_t|x_t)}{\pi_{\theta_1}^{\mathcal{S}}(u_t|x_t)}\right)\)}, which measures the similarity between policies, and the \emph{return component} \scalebox{0.9}{\(\sum_{t=1}^T \gamma^{t-1} r_t\)}, representing the discounted sum of rewards.

This approach effectively guides the agent toward states with higher expected value by considering both policy similarity and trajectory returns, ensuring that the trajectories closely resemble those in the target dataset while leveraging information from both the source and target domains. After obtaining the trajectory values, we sort all trajectories in decreasing order based on these values. 

We then sample the top \(m\) trajectories, where \(m\) is a hyperparameter that can be adjusted according to the quality of the dataset. Since these selected trajectories are the most valuable ones in terms of similarity to the target dataset, we merge them with the target dataset and use the combined data to update the critic and actor networks. Our CLTV approach is detailed in \autoref{alg:cltv}.

\begin{algorithm}[!ht]
\caption{Curriculum Learning-Based Trajectory Valuation}
\label{alg:cltv}
\begin{algorithmic}[1]
\State \textbf{Input:} Critic networks $Q_{\omega_1}^{\mathcal{S}}$, $Q_{\omega_2}^{\mathcal{T}}$, actor networks $\pi_{\theta_1}^{\mathcal{S}}$, $\pi_{\theta_2}^{\mathcal{T}}$; \\offline RL algorithm $\mathcal{A}$; \\source dataset $\mathcal{D^S}$, target dataset $\mathcal{D^T}$; \\epoch size $E$; \\training steps $S$; \\the ratio of transition score to transition reward $\lambda$
\State \textbf{Output:} Critic network $Q_{\omega_1}^{\mathcal{S}}$, actor network $\pi_{\theta_1}^{\mathcal{S}}$
\State \textbf{Initialization:}
\State Train classifiers $q_{xu}$ and $q_{xu{x}^{\prime}}$
\State Train TS model $v_{\phi}(t_j)$
\LineComment{Modify rewards in the source dataset}
\For{$j = 1, \ldots, |\mathcal{D^S}|$}
    \State $w_j = v_{\phi}(t_j)$
    \State $r_j \leftarrow (1 - \lambda) \cdot r_j + \lambda \cdot w_j$
\EndFor
\LineComment{Curriculum learning loop}
\For{$e = 1$ to $E$}
    \State $\mathcal{D} \leftarrow \mathcal{D^S}$
    \LineComment{Compute values for each trajectory in the source dataset}
    \For{each trajectory $\tau_i = \{(x_j^i, u_j^i, r_j^i)\}_{j=1}^T$ in $\mathcal{D}$}
        \State \scalebox{0.9}{$v_i = \exp\left(-\infdiv{\pi_{\theta_2}^{\mathcal{T}}(u_t|x_t)}{\pi_{\theta_1}^{\mathcal{S}}(u_t|x_t)}\right) \cdot \sum_{t=1}^{T} \gamma^{t-1} r_t$}
    \EndFor
    \LineComment{Sort trajectories by $v_i$ in decreasing order}
    \For{$s = 1$ to $S$}
        \LineComment{Sample $m$ trajectories $\{\tau_i\}_{i=1}^m$ from $\mathcal{D}$}
        \State $\mathcal{D}^{train} = \{\tau_i\}_{i=1}^m \cup \mathcal{D}^{\mathcal{T}}$
        \State Update $\omega_1, \theta_1$ with $\mathcal{A}(\mathcal{D}^{train}, Q_{\omega_1}^{\mathcal{S}}, \pi_{\theta_1}^{\mathcal{S}})$
    \EndFor
\EndFor
\end{algorithmic}
\end{algorithm}

To show the effectiveness of our approach, we first present \autoref{theorem2}, which provides a formal justification for TS by establishing bounds on the policy's performance in the target domain, and then present \autoref{theorem1}, which demonstrates how CLTV ensures value alignment with the target domain through policy improvement. Both theorems are discussed in detail in \autoref{proofs}.

\subsection{Theoretical Analysis}
\label{proofs}
In this section, we present the foundational theoretical aspects of our offline RL approach, focusing on the integration of behavior policies and the critical role of KL divergence in policy optimization. 

We begin by applying several lemmas—namely, the \emph{performance difference lemma}~\citep{agarwal2019reinforcement}, the \emph{state value estimation error}~\citep{lange2012batch}, the \emph{total variance and KL divergence relation}~\citep{csiszar2011information}, and the \emph{total variance and \(L_1\) norm relation}~\citep{IntroductiontoNonparametricEstimation}—which set the foundation for our analysis.

\begin{lemma}
\label{lemma1}
Let $\pi'$ and $\tilde{\pi}$ denote any two policies, and let $d_{\pi'}$ be the discounted state distribution induced by policy $\pi'$ over the state space $\mathcal{X}$. The following inequality holds:
\begin{equation}
\label{eq-lemma1}
\begin{aligned}
(1-\gamma)\left(V^{\pi'} - V^{\tilde{\pi}}\right) &= \mathbb{E}_{x \sim d_{\pi'}, u \sim \pi'(\cdot \mid x)}\left[A_{\tilde{\pi}}(x, u)\right] \\
&\leq \frac{2}{1-\gamma} \mathbb{E}_{x \sim d_{\pi'}}\left\|\pi'(\cdot \mid x) - \tilde{\pi}(\cdot \mid x)\right\|_1.
\end{aligned}
\end{equation}
\end{lemma}

\begin{lemma}
\label{lemma2}
Let $V^\pi$ be the value function for any policy $\pi$, and let $\hat{V}^\pi$ be an estimate of $V^\pi$. Let $r$ and $\hat{r}$ be the true reward and its estimate, respectively, both bounded in $[0,1]$. Similarly, let $p$ and $\hat{p}$ be the true and estimated transition probabilities. Then, for any discount factor $\gamma \in[0,1)$, the following inequality holds:
\begin{equation}
\label{eq-lemma2}
\begin{aligned}
\left\|V^\pi-\hat{V}^\pi\right\|_{\infty} \leq \frac{1}{1-\gamma}\left(\|r-\hat{r}\|_{\infty}+\frac{\gamma}{1-\gamma}\|p-\hat{p}\|_{\infty}\right).
\end{aligned}
\end{equation}
\end{lemma}

\begin{lemma}
\label{lemma3}
For any two probability distributions $P$ and $Q$, their total variance and KL divergence are bounded as follows, according to Pinsker’s inequality:
\begin{equation} 
\label{eq-lemma3}
{D_{\mathrm{TV}}}(P, Q) \leq \sqrt{\frac{1}{2} \infdiv{P}{Q}}.
\end{equation}
\end{lemma}

\begin{lemma}
\label{lemma4}
The total variance between any two probability distributions $P$ and $Q$ is equal to half the norm of the difference of $P$ and $Q$:
\begin{equation}
    \frac{1}{2} \| P - Q \|_1 = {D_{\mathrm{TV}}}(P, Q).
\end{equation}
\end{lemma}

\begin{corollary} 
\label{corollary2}
Let $p^{\mathcal{S}}$ and $p^{\mathcal{T}}$ represent the transition probability distributions of the source and target domains, respectively, and let $\pi$ be the policy we aim to train. The following inequality holds:
\begin{equation} 
\label{eq-corollary2}
\begin{aligned}
\left\|V_{p^{\mathcal{S}}}^{\pi}-V_{p^{\mathcal{T}}}^{\pi}\right\|_{\infty} \leq \frac{\gamma}{(1-\gamma)^2} \sqrt{\frac{1}{2} \infdiv{p^{\mathcal{S}}}{p^{\mathcal{T}}}}.
\end{aligned}
\end{equation}
\end{corollary}

\begin{proof}
We begin by applying the result from \autoref{lemma2}, noting that the source and target domains share the same reward function. Thus, we have:
\begin{equation}
\begin{aligned}
\left\| V_{p^{\mathcal{S}}}^{\pi} - V_{p^{\mathcal{T}}}^{\pi} \right\|_{\infty} \leq \frac{\gamma}{(1-\gamma)^2} \left\| p^{\mathcal{S}} - p^{\mathcal{T}} \right\|_{\infty}.
\end{aligned}
\end{equation}
To further refine this bound, we apply Pinsker's inequality (\autoref{lemma3}) which provides:
\begin{equation}
\left\| p^{\mathcal{S}} - p^{\mathcal{T}} \right\|_{\infty} \leq \sqrt{\frac{1}{2} \infdiv{p^{\mathcal{S}}}{p^{\mathcal{T}}}}.
\end{equation}
Substituting the result from Pinsker's inequality into the inequality from \autoref{lemma2}, we obtain:
\begin{equation}
\left\| V_{p^{\mathcal{S}}}^{\pi} - V_{p^{\mathcal{T}}}^{\pi} \right\|_{\infty} \leq \frac{\gamma}{(1-\gamma)^2} \sqrt{\frac{1}{2} \infdiv{p^{\mathcal{S}}}{p^{\mathcal{T}}}}.
\end{equation}
\end{proof}

\setcounter{theorem}{0}
\begin{theorem} 
\label{theorem2}
Let \( d(s) \) be a distribution over states. If the condition $\mathbb{E}_{s \sim d(s)} \left[ \sqrt{\frac{1}{2} \infdiv{p^{\mathcal{S}}}{p^{\mathcal{T}}}} \right] \leq \epsilon$ is satisfied, the expected value of the policy \(\pi\) in the target domain can be bounded below by:
\begin{equation} 
\label{eq-theorem2}
\mathbb{E}_{s \sim d(s)} \left[ V_{p^{\mathcal{T}}}^{\pi} \right] \geq \mathbb{E}_{s \sim d(s)} \left[ V_{p^{\mathcal{S}}}^{\pi} \right] - \frac{\gamma \epsilon}{(1-\gamma)^2}.
\end{equation}
\end{theorem}

\begin{proof}
We begin by applying the expectation operator to \autoref{corollary2}, which provides:
\begin{equation} 
\label{eq1-theorem2-proof}
\resizebox{0.9\hsize}{!}{$
\mathbb{E}_{s \sim d(s)} \left[ \left\| V_{p^{\mathcal{S}}}^{\pi} - V_{p^{\mathcal{T}}}^{\pi} \right\|_{\infty} \right] \leq \frac{\gamma}{(1-\gamma)^2} \mathbb{E}_{s \sim d(s)} \left[ \sqrt{\frac{1}{2}\infdiv{p^{\mathcal{S}}}{p^{\mathcal{T}}}} \right].
$}
\end{equation}
By the non-negativity of expectations, we have:
\begin{equation} 
\label{eq2-theorem2-proof}
\mathbb{E}_{s \sim d(s)} \left[ V_{p^{\mathcal{S}}}^{\pi} - V_{p^{\mathcal{T}}}^{\pi} \right] \leq \mathbb{E}_{s \sim d(s)} \left[ \left\| V_{p^{\mathcal{S}}}^{\pi} - V_{p^{\mathcal{T}}}^{\pi} \right\|_{\infty} \right].
\end{equation}
Substituting the condition into \autoref{eq1-theorem2-proof}, we obtain:
\begin{equation} 
\label{eq3-theorem2-proof}
\mathbb{E}_{s \sim d(s)} \left[ V_{p^{\mathcal{S}}}^{\pi} - V_{p^{\mathcal{T}}}^{\pi} \right] \leq \frac{\gamma \epsilon}{(1-\gamma)^2}.
\end{equation}
Consequently, this yields the lower bound:
\begin{equation} 
\label{eq4-theorem2-proof}
\mathbb{E}_{s \sim d(s)} \left[ V_{p^{\mathcal{T}}}^{\pi} \right] \geq \mathbb{E}_{s \sim d(s)} \left[ V_{p^{\mathcal{S}}}^{\pi} \right] - \frac{\gamma \epsilon}{(1-\gamma)^2}.
\end{equation}
\end{proof}

\begin{corollary}
\label{corollary1}
Let \(\pi_b\) denote the behavior policy, which induces a distribution \(d_{\pi_b}\) over the state space \(\mathcal{X}\). For the current policy $\pi_i$, and the policy in the next episode $\pi_{i+1}$, the following relationship is valid:
\begin{equation} 
\label{eq-corollary1}
\begin{aligned}
& \mathbb{E}_{x \sim d_{\pi_b}, u \sim \pi_b(\cdot \mid x)} \bigg[ A^{\pi_i}(x, u) \bigg] \\
& \quad\; -\frac{2}{1-\gamma} \mathbb{E}_{x \sim d_{\pi_b}} D_{\mathrm{TV}}\left(\pi_b^{}(\cdot \mid x), \pi_{i+1}(\cdot \mid x)\right) \\
& \leq (1-\gamma)\left(V^{\pi_{i+1}}-V^{\pi_i}\right).
\end{aligned}
\end{equation}
\end{corollary}

\begin{proof}
By substituting $\pi^*$ with $\pi_{b}$ and $\tilde{\pi}$ with $\pi_{i+1}$ in \autoref{eq-lemma1}, we obtain the first part of the inequality. 
\begin{equation}
\resizebox{0.9\hsize}{!}{$
(1-\gamma)(V^{\pi_b}-V^{{\pi_{i+1}}}) \leq \frac{1}{1-\gamma} \mathbb{E}_{x \sim d_{\pi_b}} \| {\pi_b}(\cdot \mid x)-{\pi_{i+1}}(\cdot \mid x) \|_1
$}
\label{eq1-corollary-proof}
\end{equation}

Similarly, by substituting $\pi^*$ with $\pi_{b}$ and $\tilde{\pi}$ with $\pi_{i}$ in \autoref{eq-lemma1}, we derive \autoref{eq2-corollary-proof}, which represents the second part of the inequality in \autoref{eq-lemma1}.
\begin{equation}
\resizebox{0.8\hsize}{!}{$
(1-\gamma)(V^{\pi_b}-V^{{\pi_{i}}}) = \mathbb{E}_{x \sim d_{\pi_b}, u \sim \pi_b(\cdot \mid x)} \bigg[ A^{\pi_i}(x, u) \bigg]
$}
\label{eq2-corollary-proof}
\end{equation}

By applying \autoref{eq1-corollary-proof} and \autoref{eq2-corollary-proof}, we derive:
\begin{equation}
\resizebox{0.9\hsize}{!}{$
\begin{aligned}
(1-\gamma)(V^{\pi_{i+1}}-V^{{\pi_{i}}}) & \geq \mathbb{E}_{x \sim d_{\pi_b}, u \sim \pi_b(\cdot \mid x)} \bigg[ A^{\pi_i}(x, u) \bigg] \\
& \quad\; -\frac{1}{1-\gamma} \mathbb{E}_{x \sim d_{\pi_b}} \| {\pi_b}(\cdot \mid x)-{\pi_{i+1}}(\cdot \mid x) \|_1
\end{aligned}
$}
\label{eq3-corollary-proof}
\end{equation}

Utilizing \autoref{lemma4}, we derive:
\begin{equation}
\resizebox{0.9\hsize}{!}{$
\begin{aligned}
(1-\gamma)(V^{\pi_{i+1}}-V^{{\pi_{i}}}) & \geq \mathbb{E}_{x \sim d_{\pi_b}, u \sim \pi_b(\cdot \mid x)} \bigg[ A^{\pi_i}(x, u) \bigg] \\
& \quad\; -\frac{2}{1-\gamma} \mathbb{E}_{x \sim d_{\pi_b}} D_{\mathrm{TV}}\left(\pi_b^{}(\cdot \mid x), \pi_{i+1}(\cdot \mid x)\right)
\end{aligned}
$}
\label{eq4-corollary-proof}
\end{equation}
\end{proof}

\begin{theorem}
\label{theorem1}
The KL divergence between the behavior policy $\pi_b$ and the policy in the subsequent episode $\pi_{i+1}$ establishes a lower bound for policy improvement:
\begin{equation} 
\label{eq-theorem1}
\resizebox{0.7\hsize}{!}{$
\begin{aligned}
& (1-\gamma)\left(V^{\pi_{i+1}}-V^{\pi_i}\right) \\
& \geq \mathbb{E}_{x \sim d_{\pi_b}, u \sim \pi_b(\cdot \mid x)} \bigg[ A^{\pi_i}(x, u) \bigg] \\
& \quad\; -\frac{\sqrt{2}}{1-\gamma} \mathbb{E}_{x \sim d_{\pi_b}} \sqrt{\infdiv{\pi_b(\cdot \mid x)}{\pi_{i+1}(\cdot \mid x)}}.
\end{aligned}
$}
\end{equation}
\end{theorem}

\begin{proof}
In \autoref{corollary1}, we derived bounds for policy improvement by introducing the total variation distance. To relate total variation to KL divergence, we utilize \autoref{lemma3}, which gives us Pinsker's inequality:

\begin{equation}
\label{eq1-theorem1-proof}
\resizebox{0.48\hsize}{!}{
\begin{math}
\begin{aligned}
    {D_{\mathrm{TV}}}(P, Q) & \leq \sqrt{\frac{1}{2} \infdiv{P}{Q}}.
\end{aligned}
\end{math}
}
\end{equation}

Next, we substitute the total variation distance in \autoref{eq4-corollary-proof} with the KL divergence bound using \autoref{eq1-theorem1-proof}:
\begin{equation}
\resizebox{0.9\hsize}{!}{$
\begin{aligned} 
\label{eq2-theorem1-proof}
(1-\gamma)(V^{\pi_{i+1}}-V^{{\pi_{i}}}) & \geq \mathbb{E}_{x \sim d_{\pi_b}, u \sim \pi_b(\cdot \mid x)} \bigg[ A^{\pi_i}(x, u) \bigg] \\
& \quad\; -\frac{2}{1-\gamma} \mathbb{E}_{x \sim d_{\pi_b}} D_{\mathrm{TV}}\left(\pi_b(\cdot \mid x), \pi_{i+1}(\cdot \mid x)\right) \\
& \geq \mathbb{E}_{x \sim d_{\pi_b}, u \sim \pi_b(\cdot \mid x)} \bigg[ A^{\pi_i}(x, u) \bigg] \\
& \quad\; -\frac{2}{1-\gamma} \mathbb{E}_{x \sim d_{\pi_b}} \sqrt{\frac{1}{2}\infdiv{\pi_b(\cdot \mid x)}{\pi_{i+1}(\cdot \mid x)}} \\
& = \mathbb{E}_{x \sim d_{\pi_b}, u \sim \pi_b(\cdot \mid x)} \bigg[ A^{\pi_i}(x, u) \bigg] \\
& \quad\; -\frac{\sqrt{2}}{1-\gamma} \mathbb{E}_{x \sim d_{\pi_b}} \sqrt{\infdiv{\pi_b(\cdot \mid x)}{\pi_{i+1}(\cdot \mid x)}}.
\end{aligned}
$}
\end{equation}
\end{proof}

\section{Experiments}
\label{sec:experiments}
\textbf{Datasets.}  We used D4RL~\citep{fu2020d4rl} datasets to create three mixed datasets for each of the considered MuJoCo~\citep{todorov2012mujoco} domains: Ant, HalfCheetah, Hopper, and Walker2d. These datasets were created by mixing 90\% of transition items from the source datasets with 10\% from the target datasets. In D4RL, the \textit{medium} dataset is generated by first training a policy online using Soft Actor-Critic (SAC)~\citep{haarnoja2018soft}, early-stopping the training, and collecting 1M transition items from this partially trained policy. The \textit{random} datasets are generated by unrolling a randomly initialized policy in these domains. Similarly, the \textit{expert} datasets are generated by an expert policy~\citep{fu2020d4rl}.

\vspace{0.5em}
\noindent\textbf{Baselines.} We use CQL~\citep{kumar2020conservative} and IQL~\citep{kostrikov2021offline} as our base offline RL algorithms as they are the most widely used offline RL algorithms and have demonstrated good results in tasks similar to ours, making them reliable benchmarks~\citep{fujita2021chainerrl,seno2022d3rlpy,sun2023offlinerl}. 

We consider Vanilla (the base algorithm without any additional methods applied), CUORL~\citep{cai2023curriculum}, and Harness~\citep{hong2023harnessing} as our \textit{baselines}, applied on top of the base offline RL algorithms (CQL and IQL). 

\vspace{0.5em}
\noindent\textbf{Performance Metric.} We measure the performance of offline RL algorithms using a normalized score, as introduced in D4RL~\citep{fu2020d4rl}. The score is calculated as: \scalebox{0.85}{$100 \times \frac{\text {score} - \text{random score}}{\text {expert score} - \text{random score}}$}. This formula is used to normalize the scores for each domain, roughly scaling them to a range between 0 and 100. A normalized score of 0 corresponds to the average returns (over 100 episodes, with each episode containing 5,000 steps) of an agent that takes actions uniformly at random across the action space. A score of 100 corresponds to the average returns of a domain expert. Scores below 0 indicate performance worse than random, while scores above 100 indicate performance surpassing expert level.

\vspace{0.5em}
\noindent\textbf{Ablation Study.} The choice of reward function and the impact of similarity-reward trade-off parameters of our method (\(\lambda\) and \(\delta\)) are discussed in \autoref{rew-choice} and \autoref{par-impact}, respectively.

\vspace{0.5em}
\noindent\textbf{Implementation and Parameter Tuning.} The implementation details and hyperparameter tuning are presented in \autoref{ex-details}. Our code
is available at \href{https://github.com/amir-abolfazli/CLTV}{https://github.com/amir-abolfazli/CLTV}.

\subsection{Performance of Offline RL Methods}
\label{subsec-score}
\autoref{tab:normalized-score} provides a comprehensive evaluation of three methods (CUORL, Harness, and CLTV) alongside the Vanilla version, using two popular offline RL algorithms (CQL and IQL). The evaluation reports their performance, measured by normalized score, and standard deviations over 100 episodes and 5 seeds on mixed D4RL datasets. The highest performing scores are highlighted in blue, and the total scores of the second best performing method are marked in bold. Additionally, the percentage increase (PI) in score of CLTV, compared to the second best performing method, is reported. The runtime analysis is presented in \autoref{runtime-analysis}.

The results show that CLTV consistently outperforms the other methods (Vanilla, CUORL and Harness) across all domains and algorithms. For example, in the Ant domain using CQL, CLTV achieves a total score of 238.18, which is 95\% higher than the second-best method (Harness) with a score of 120.97. Similarly, in the IQL variant for the Walker2d domain, CLTV surpasses the second-best method by 28\%, scoring 268.62 against 208.80. Across the HalfCheetah and Hopper domains, CLTV demonstrates significant improvements, with percentage increases of 43\% and 62\% respectively, underscoring its superior performance as reflected by the relatively low standard deviations. Moreover, CLTV often highlights the highest individual dataset scores, indicating its robustness across different datasets within each domain.

\begin{table}[!ht]
\caption{Average normalized score along with standard deviation over 100 episodes and 5 seeds on mixed D4RL datasets.}
\label{tab:normalized-score}
\centering
\resizebox{\linewidth}{!}{
\renewcommand{\arraystretch}{1.4}
\begin{tabular}{@{}ccclllll@{}}
\toprule
\textbf{Domain} & \textbf{\shortstack{RL \\ Algorithm}} & \textbf{Method} & \multicolumn{3}{c}{\textbf{Dataset}} & \textbf{Total} & \textbf{PI} \\
\multicolumn{1}{l}{} & \multicolumn{1}{l}{} & \multicolumn{1}{l}{} & \multicolumn{1}{c}{\textbf{random-medium}} & \textbf{random-expert} & \textbf{medium-expert} &  &  \\
\hline
\multirow{9}{*}{\rotatebox{90}{\Large Ant}} & \multirow{4}{*}{CQL} & Vanilla & 51.99 ± 22.32  & 49.81 ± 25.85 & -6.25 ± 14.12 & 95.55 &  \\
 &  & CUORL & -8.37 ± 2.59 & -7.81 ± 4.61 & 6.32 ± 23.51  & -9.86 &  \\
 &  & Harness & 67.19 ± 6.47 & 38.30 ± 38.32 & 15.48 ± 22.20 & \textbf{120.97} &  \\
 &  & CLTV & \paddedcolorbox{LightCyan}{97.48} ± 3.63 & \paddedcolorbox{LightCyan}{115.86} ± 6.82 & \paddedcolorbox{LightCyan}{24.84} ± 13.70  & \paddedcolorbox{LightCyan}{238.18} & $\uparrow$ 95\% \\ \cmidrule{3-8}
 & \multirow{4}{*}{IQL} & Vanilla & \paddedcolorbox{LightCyan}{83.75} ± 7.18 & 82.96 ± 9.07 & \paddedcolorbox{LightCyan}{117.20} ± 4.39 & \textbf{283.91} &  \\
 &  & CUORL & 4.60 ± 1.18 & 4.66 ± 0.49  & 116.23 ± 2.50 & 125.49 &  \\
 &  & Harness & 74.45 ± 6.52 & 82.56 ± 8.26 & 110.57 ± 2.56 & 267.58 &  \\
 &  & CLTV & 78.67 ± 8.26 & \paddedcolorbox{LightCyan}{88.26} ± 4.67 & 117.00 ± 7.28 & \paddedcolorbox{LightCyan}{283.93} & $\uparrow$ 0\% \\ 
  \hline
 \multirow{9}{*}{\rotatebox{90}{\Large HalfCheetah}} & \multirow{4}{*}{CQL} & Vanilla & 35.15 ± 2.78 & -0.09 ± 0.75 & 29.17 ± 18.12 & 64.23 & \\
 &  & CUORL & -3.42 ± 0.46 & -3.47 ± 0.30 & 36.89 ± 16.87 & 30.00 &  \\
 &  & Harness & 34.79 ± 4.12 & 0.53 ± 1.41 & 44.19 ± 10.86 & \textbf{79.51} &  \\
 &  & CLTV & \paddedcolorbox{LightCyan}{44.13} ± 3.47 & \paddedcolorbox{LightCyan}{10.37} ± 2.51 & \paddedcolorbox{LightCyan}{59.88} ± 7.91 & \paddedcolorbox{LightCyan}{114.38} & $\uparrow$ 43\% \\ \cmidrule{3-8}
 & \multirow{4}{*}{IQL} & Vanilla & 37.78 ± 1.51 &  6.97 ± 2.33 & 58.05 ± 3.48 & \textbf{102.80} &  \\
 &  & CUORL & 2.97 ± 1.50 & 2.55 ± 1.40 & 60.53 ± 4.25 & 66.05 &  \\
 &  & Harness & 36.55 ± 2.78 & 8.37 ± 2.80 & 55.90 ± 3.22 & 100.82 &  \\
 &  & CLTV & \paddedcolorbox{LightCyan}{41.83} ± 0.63 & \paddedcolorbox{LightCyan}{16.28} ± 7.22 & \paddedcolorbox{LightCyan}{77.10} ± 5.16 & \paddedcolorbox{LightCyan}{135.21} & $\uparrow$ 31\%  \\ 
  \hline
\multirow{9}{*}{\rotatebox{90}{\Large Hopper}} & \multirow{4}{*}{CQL} & Vanilla & 18.93 ± 12.57 & 22.47 ± 23.90 & 75.90 ± 11.12 & \textbf{117.30} &  \\
 &  & CUORL & 0.83 ± 0.23 & 0.73 ± 0.15 & 71.36 ± 40.45 & 72.90 &  \\
 &  & Harness & 12.90 ± 14.26 & 22.63 ± 7.59 & 65.33 ± 29.68 & 100.86 &  \\
 &  & CLTV & \paddedcolorbox{LightCyan}{51.04} ± 3.21 & \paddedcolorbox{LightCyan}{56.23} ± 15.14 & \paddedcolorbox{LightCyan}{83.44} ± 16.95 & \paddedcolorbox{LightCyan}{190.71} & $\uparrow$ 62\% \\ \cmidrule{3-8}
 & \multirow{4}{*}{IQL} & Vanilla & 53.29 ± 2.71 & 28.30 ± 6.21 & 25.63 ± 16.76 & 107.22 &  \\
 &  & CUORL & 6.70 ± 3.80 & 14.91 ± 12.51 & 12.59 ± 10.06 & 34.20 &  \\
 &  & Harness & 55.17 ± 3.55 & 32.05 ± 9.78 & 24.79 ± 12.61 & \textbf{112.01} &  \\
 &  & CLTV & \paddedcolorbox{LightCyan}{55.79} ± 4.44 & \paddedcolorbox{LightCyan}{39.56} ± 2.78 & \paddedcolorbox{LightCyan}{70.73} ± 4.11 & \paddedcolorbox{LightCyan}{166.08} & $\uparrow$ 48\% \\ 
  \hline
 \multirow{9}{*}{\rotatebox{90}{\Large Walker2d}} & \multirow{4}{*}{CQL} & Vanilla & 27.55 ± 20.29 & 83.81 ± 19.68 & 2.94 ± 3.80 & \textbf{114.30} &  \\
 &  & CUORL & -0.01 ± 0.09 & -0.03 ± 0.07 & 8.22 ± 9.01 & 8.18 &  \\
 &  & Harness & 23.70 ± 15.89 & 87.04 ± 17.11 & 1.77 ± 1.35 & 112.51 &  \\
 &  & CLTV & \paddedcolorbox{LightCyan}{70.45} ± 8.72 & \paddedcolorbox{LightCyan}{97.43} ± 7.74 & \paddedcolorbox{LightCyan}{30.23} ± 41.25 & \paddedcolorbox{LightCyan}{198.11} & $\uparrow$ 73\% \\ \cmidrule{3-8}
 & \multirow{4}{*}{IQL} & Vanilla & 63.98 ± 5.90 & 45.67 ± 10.30 & 96.65 ± 5.87 & 206.30 &  \\
 &  & CUORL & 5.05 ± 1.94 & 3.58 ± 0.72 & 88.71 ± 5.00 & 97.34 &  \\
 &  & Harness & 68.20 ± 3.17 & 45.12 ± 14.86 & 95.48 ± 3.81 & \textbf{208.80} &  \\
 &  & CLTV & \paddedcolorbox{LightCyan}{68.37} ± 4.12 & \paddedcolorbox{LightCyan}{89.51} ± 9.03 & \paddedcolorbox{LightCyan}{110.74} ± 0.66 & \paddedcolorbox{LightCyan}{268.62} & $\uparrow$ 28\% \\ 
\bottomrule
\end{tabular}
}
\end{table}

\subsection{Impact of CL on Performance of CLTV}
\begin{figure}[!ht]
\center
\includegraphics[width=0.497\textwidth,keepaspectratio]{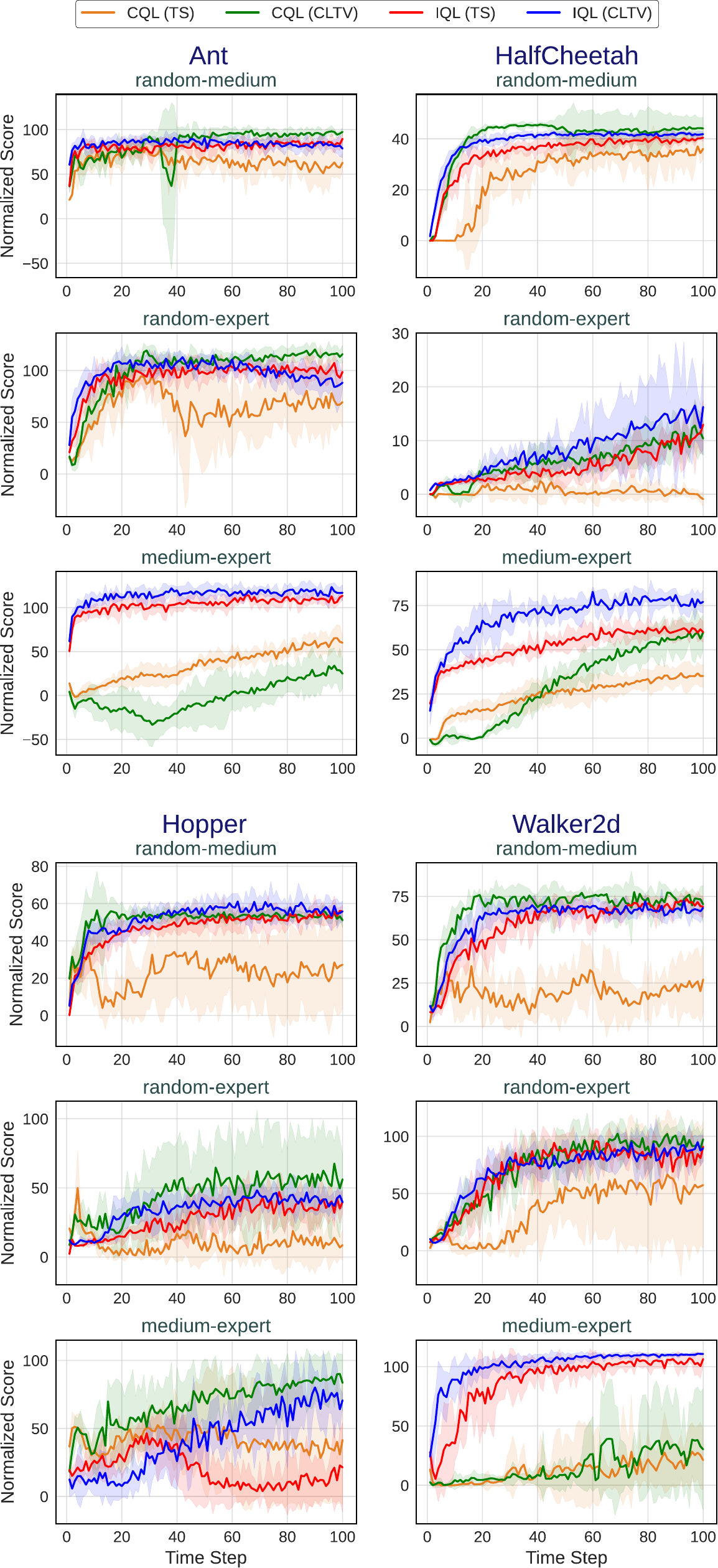}
\caption{Performance of TS method, compared with CLTV. The curves are averaged over 5 seeds, with the shaded areas showing the confidence interval across seeds.}
\label{fig-wuthoutCL}
\end{figure}

Our method consists of two key components: Transition Scoring (TS) and Curriculum Learning-Based Trajectory Valuation (CLTV). TS assigns scores to transitions and adjusts rewards, enabling the agent to focus on the transitions most relevant to the target domain. Building upon TS, CLTV integrates curriculum learning (CL) to prioritize high-quality trajectories during training.

\autoref{fig-wuthoutCL} illustrates the impact of CL on the performance of CLTV. The results demonstrate that while TS provides a strong foundation for improving agent performance, CLTV consistently achieves better results across various environments, including faster convergence and higher normalized scores. For example, in the Ant domain, CLTV achieves scores approaching or exceeding 100 in the random-medium and random-expert datasets, whereas TS reaches around 50. Similarly, in HalfCheetah, Hopper, and Walker2d, CLTV builds on the framework of TS to deliver competitive results, especially in the random-medium and medium-expert datasets. Additionally, CLTV tends to reduce variance and improve stability in some environments, as evidenced by the smoother learning curves for Walker2d and Ant.

Although performance improvements in HalfCheetah and Hopper are less pronounced, CLTV consistently surpasses TS in most cases, highlighting its robustness across a wide range of conditions.

In general, CLTV improves the learning performance of TS, particularly in settings with diverse data distributions. By prioritizing high-value trajectories, CLTV enables the agent to focus on relevant experiences, resulting in more efficient learning and better adaptability to variations in data quality.

In general, CLTV improves TS by prioritizing high-value trajectories, leading to faster convergence and improved stability. Its adaptability across diverse environments makes it a robust approach for optimizing trajectory valuation in offline RL.

\section{Conclusion}
\label{sec:conclusion}
In this work, we introduced Curriculum Learning-Based Trajectory Valuation (CLTV) to enhance the performance of offline RL algorithms when dealing with mixed datasets, characterized by a predominance of source data generated from random or suboptimal policies and a limited amount of target data from higher-quality policies. Our approach employs Transition Scoring (TS) to evaluate transition items based on their relevance to the target domain. By leveraging these scores, CLTV prioritizes high-quality trajectories through curriculum learning, enabling the agent to utilize valuable transitions from datasets generated by diverse policies. Experimental results across various algorithms and MuJoCo environments demonstrated that CLTV significantly improves performance and accelerates convergence. CLTV is particularly effective in batch-constrained scenarios, operating efficiently with limited data to ensure robust learning even when data is scarce. Furthermore, our theoretical analysis validated the effectiveness of our approach.


\begin{acks}
This work was supported by the research projects ``QuBRA'' and ``BIFOLD'', funded by the Federal Ministry of Education and Research (BMBF) under grant IDs 13N16052 and BIFOLD24B, respectively.
\end{acks}


\bibliographystyle{ACM-Reference-Format} 
\bibliography{references}

\clearpage
\newpage

\appendix
\section*{Appendix}
\section{Theoretical Analysis}
\label{derivation}
\subsection*{Derivation of \autoref{grad-eq}}
The gradient of the objective function \( J\left(\pi_\phi\right) \) (\autoref{objective}) with respect to the policy parameters \( \phi \) is given by:

\begin{equation}
\resizebox{0.8\hsize}{!}{%
$\displaystyle
\begin{aligned}
\nabla_{\phi} J\left(\pi_{\phi}\right) &= \int P^{\mathcal{S}}\left(x^{\mathcal{S}}, u^{\mathcal{S}}, x^{\prime \mathcal{S}}\right) \Bigg[ \sum_{w \in [0,1]^{N}} \nabla_{\phi} \pi_{\phi}(\mathcal{D_S}, w) \cdot r_\phi\left(x^{\mathcal{S}}, u^{\mathcal{S}}, x^{\prime \mathcal{S}}, \Delta_\theta\right) \\
&\quad + \sum_{w \in [0,1]^{N}} \pi_{\phi}(\mathcal{D_S}, w) \cdot \nabla_{\phi} r_\phi\left(x^{\mathcal{S}}, u^{\mathcal{S}}, x^{\prime \mathcal{S}}, \Delta_\theta\right) \Bigg] \, d\left(x^{\mathcal{S}}, u^{\mathcal{S}}, x^{\prime \mathcal{S}}\right),
\end{aligned}$
}
\end{equation}

where \( P^{\mathcal{S}}\left(x^{\mathcal{S}}, u^{\mathcal{S}}, x^{\prime \mathcal{S}}\right) \) represents the probability distribution over state--action--next-state triples, \( \pi_{\phi}(\mathcal{D_S}, w) \) is the policy parameterized by \( \phi \), and \( r_{\phi} \) is the reward function.

We apply the log-derivative trick to the policy gradient:

\begin{equation}
\resizebox{0.6\hsize}{!}{%
$\displaystyle
\nabla_{\phi} \pi_{\phi}(\mathcal{D_S}, w) = \pi_{\phi}(\mathcal{D_S}, w) \nabla_{\phi} \log \pi_{\phi}(\mathcal{D_S}, w).
$
}
\end{equation}

Substituting this into the integral, we obtain:

\begin{equation}
\resizebox{0.98\hsize}{!}{%
$\displaystyle
\begin{aligned}
\nabla_{\phi} J(\pi_{\phi}) &= \int P^{\mathcal{S}}\left(x^{\mathcal{S}}, u^{\mathcal{S}}, x^{\prime \mathcal{S}}\right) \Bigg[ \sum_{w \in [0,1]^{N}} \pi_{\phi}(\mathcal{D_S}, w) \nabla_{\phi} \log \pi_{\phi}(\mathcal{D_S}, w) \cdot r_\phi\left(x^{\mathcal{S}}, u^{\mathcal{S}}, x^{\prime \mathcal{S}}, \Delta_\theta\right) \\
&\quad + \sum_{w \in [0,1]^{N}} \pi_{\phi}(\mathcal{D_S}, w) \cdot \nabla_{\phi} r_\phi\left(x^{\mathcal{S}}, u^{\mathcal{S}}, x^{\prime \mathcal{S}}, \Delta_\theta\right) \Bigg] \, d\left(x^{\mathcal{S}}, u^{\mathcal{S}}, x^{\prime \mathcal{S}}\right).
\end{aligned}$
}
\end{equation}

The sum over \( w \) can be interpreted as an expectation with respect to the policy distribution \( \pi_{\phi} \). Therefore, we simplify the expression to:

\begin{equation}
\resizebox{0.9\hsize}{!}{%
$\displaystyle
\begin{aligned}
\nabla_{\phi} J(\pi_{\phi}) &= \int P^{\mathcal{S}}\left(x^{\mathcal{S}}, u^{\mathcal{S}}, x^{\prime \mathcal{S}}\right) \, \mathbb{E}_{w \sim \pi_{\phi}(\mathcal{D_S}, \cdot)} \Bigg[ \nabla_{\phi} \log \pi_{\phi}(\mathcal{D_S}, w) \cdot r_\phi\left(x^{\mathcal{S}}, u^{\mathcal{S}}, x^{\prime \mathcal{S}}, \Delta_\theta\right) \\
&\quad + \nabla_{\phi} r_\phi\left(x^{\mathcal{S}}, u^{\mathcal{S}}, x^{\prime \mathcal{S}}, \Delta_\theta\right) \Bigg] \, d\left(x^{\mathcal{S}}, u^{\mathcal{S}}, x^{\prime \mathcal{S}}\right).
\end{aligned}$
}
\end{equation}

Finally, by interpreting the integral over \( P^{\mathcal{S}} \) and the expectation over \( \pi_{\phi} \) jointly as an expectation with respect to the distribution of trajectories under the current policy, we arrive at the final expression:

\begin{equation}
\resizebox{0.67\hsize}{!}{%
$\displaystyle
\begin{aligned}
    \nabla_{\phi} J\left(\pi_{\phi}\right) 
    &= \mathbb{E}_{\substack{(x^{\mathcal{S}}, u^{\mathcal{S}}, x^{\prime \mathcal{S}}) \sim P^{\mathcal{S}}\\ w \sim \pi_{\phi}(\mathcal{D_S}, \cdot)}} \Bigg[ r_\phi\left(x^{\mathcal{S}}, u^{\mathcal{S}}, x^{\prime \mathcal{S}}, \Delta_\theta\right) \cdot \nabla_{\phi} \log \pi_{\phi}(\mathcal{D_S}, w) \\
    & \quad + \nabla_{\phi} r_\phi\left(x^{\mathcal{S}}, u^{\mathcal{S}}, x^{\prime \mathcal{S}}, \Delta_\theta\right) \Bigg].
\end{aligned}$
}
\end{equation}
\section{Additional Experimental Results}
\subsection{Runtime Analysis of Offline RL Methods}
\label{runtime-analysis}
\autoref{fig-runtime} shows the runtimes of CLTV compared to Vanilla, CUORL, and Harness across different datasets and offline RL algorithms.

In the Ant domain, CLTV has a higher runtime compared to CQL and IQL. A similar pattern is observed in the Hopper and Walker2d domains, where CLTV’s runtime is generally higher than the other methods across most cases. However, the gains achieved in key learning tasks, particularly in the expert-level datasets, justify the additional computational cost. In the HalfCheetah domain, the runtime differences between CQL and IQL methods are smaller, but CLTV still incurs higher computational overhead compared to other methods, particularly in the random-medium setting. 

In conclusion, while CLTV may not always achieve the shortest runtimes, the performance improvements it provides (as reported in \autoref{tab:normalized-score}) justify the additional computational time. Across different environments and datasets, CLTV offers a reasonable trade-off between runtime and learning performance, making it a practical and effective option for offline RL tasks.

\begin{figure}[!ht]
\center
\includegraphics[width=0.32\textwidth]{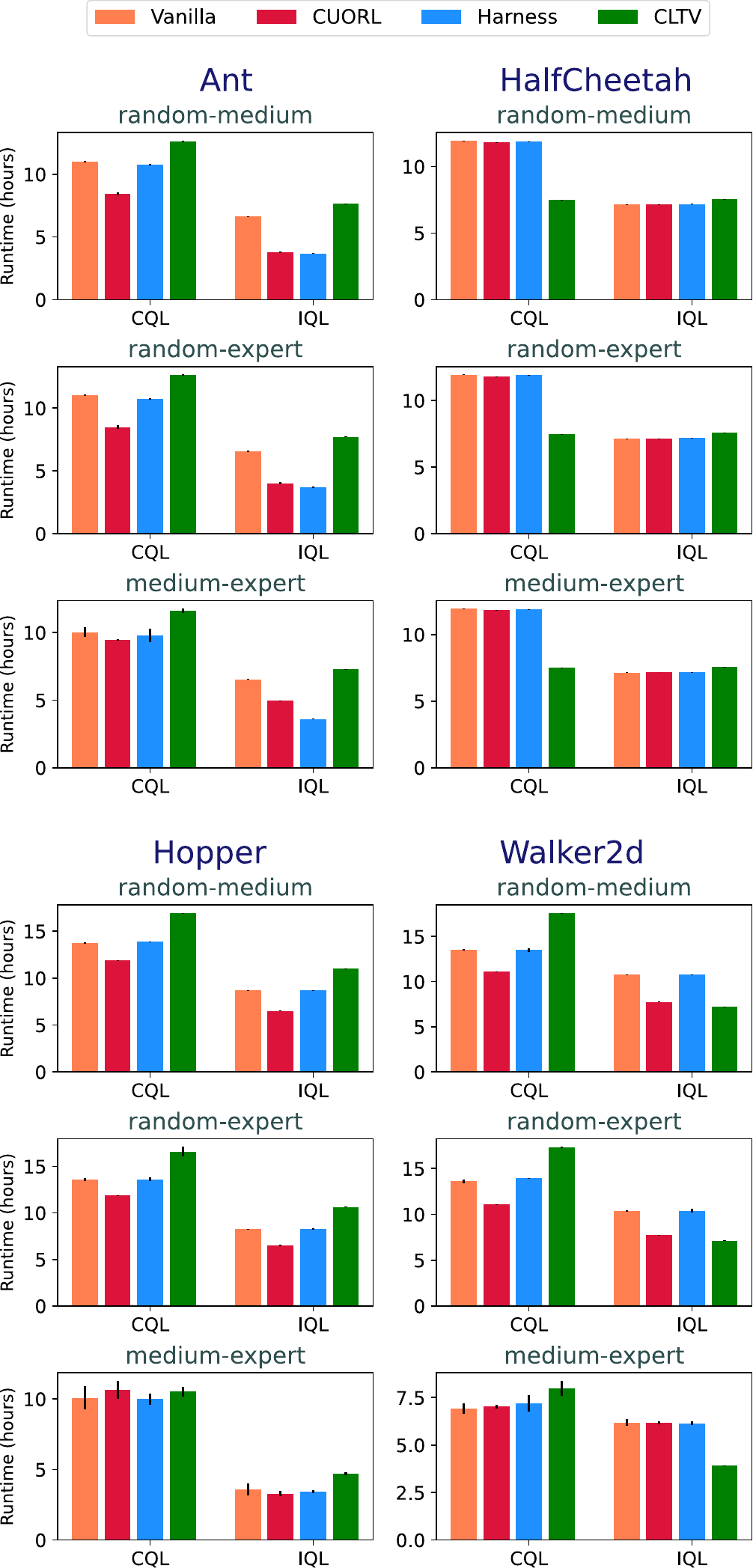}
\caption{Runtime analysis of offline RL algorithms.}
\label{fig-runtime}
\end{figure}
\subsection{Choice of Reward Function in CLTV}
\label{rew-choice}
To demonstrate the effectiveness of our reward function, we conducted two comparative experiments: one using the temporal difference (TD)~\citep{tdl} and the other using the reward shaping (RS)~\citep{pbrs} to modify the rewards of transitions. 

The results of the comparison are presented in \autoref{tab:reward-functions}, where CLTV-TD corresponds to the CLTV model using TD as the reward function (\( r_{TD} \)), while CLTV-RS uses RS as the reward function (\( r_{RS} \)).

The TD reward function is presented in \autoref{td-rf}:

\begin{equation}
\label{td-rf}
r_{\mathrm{TD}}=V(x)+\alpha\left(r+\gamma V\left(x^{\prime}\right)-V(x)\right)
\end{equation}

Similarly, the RS reward function is presented in \autoref{rs-rf}:

\begin{equation}
\label{rs-rf}
r_{\mathrm{RS}}=\gamma V\left(x^{\prime}\right)-V(x)
\end{equation}

\autoref{tab:reward-functions} compares the performance of CLTV with its variants, CLTV-TD and CLTV-RS, across all considered environments and datasets, using CQL and IQL as base algorithms.

In the Ant domain, CLTV achieves higher normalized scores with both CQL and IQL algorithms. For example, on the random-medium dataset, CLTV reaches a score of \(97.48\), compared to CLTV-TD's \(73.36\) and CLTV-RS's \(73.50\). Similarly, in the random-expert dataset, CLTV attains \(115.86\), outperforming the other methods. 

In the HalfCheetah domain, CLTV also exhibits superior performance, particularly on the medium-expert dataset, where it achieves \(59.88\) using the CQL algorithm. This result is higher than the scores of \(0.50\) and \(28.31\) recorded for CLTV-TD and CLTV-RS, respectively. For the IQL algorithm, CLTV performs well, reaching a score of \(77.10\) and clearly surpassing the other methods. 

The Walker2d and Hopper domains further highlight the advantages of CLTV. In the random-expert dataset of Walker2d, CLTV achieves \(97.43\) using CQL, outperforming both variants by a substantial margin. Similarly, in the Hopper domain, CLTV outperforms its counterparts on the medium-expert dataset, achieving \(83.44\) compared to the lower scores of the TD and RS variants. 

The results show that our reward function outperforms the other two in almost all cases. This suggests that our reward function is particularly effective for domain transfer learning.

\begin{table}[!ht]
\caption{{Performance of our CLTV method compared to two variants: CLTV-TD, which uses temporal difference (TD), and CLTV-RS, which uses reward shaping (RS). Normalized scores with standard deviations over 100 episodes and 5 seeds on mixed D4RL datasets are reported using the base algorithms CQL and IQL. The highest scores are highlighted in blue.}}
\label{tab:reward-functions}
\centering
\resizebox{\linewidth}{!}{
\renewcommand{\arraystretch}{1.4}
\begin{tabular}{@{}ccllll@{}}
\toprule
\textbf{Domain} & \textbf{\shortstack{RL \\ Algorithm}} & \textbf{Method} & \multicolumn{3}{c}{\textbf{Dataset}}\\
\multicolumn{1}{l}{} & \multicolumn{1}{l}{} & \multicolumn{1}{l}{} & \multicolumn{1}{c}{\textbf{random-medium}} & \textbf{random-expert} & \textbf{medium-expert}\\
\hline
\multirow{6}{*}{\rotatebox{90}{\Large Ant}} & \multirow{3}{*}{CQL} & CLTV-TD &  73.36 ± 6.72 & 95.08 ± 23.14 & 1.17 ± 46.66 \\
 &  & CLTV-RS & 73.50 ± 4.48 & 62.98 ± 12.38 & 6.48 ± 26.03 \\
 &  & CLTV & \paddedcolorbox{LightCyan}{97.48} ± 3.63 & \paddedcolorbox{LightCyan}{115.86} ± 6.82 & \paddedcolorbox{LightCyan}{24.84} ± 13.70 \\ \cmidrule{3-6}
 & \multirow{3}{*}{IQL} & CLTV-TD & 18.78 ± 8.78 & 7.87 ± 1.94 & 113.07 ± 3.89 \\
 &  & CLTV-RS & 75.11 ± 7.19 & 50.80 ± 11.16 & 111.82 ± 3.99 \\
 &  & CLTV & \paddedcolorbox{LightCyan}{78.67} ± 8.26 & \paddedcolorbox{LightCyan}{88.26} ± 4.67 & \paddedcolorbox{LightCyan}{117.00} ± 7.28 \\ 
  \hline
\multirow{6}{*}{\rotatebox{90}{\Large HalfCheetah}} & \multirow{3}{*}{CQL} & CLTV-TD & 36.87 ± 1.73 & 2.69 ± 2.20 & 0.50 ± 3.82 \\
 &  & CLTV-RS & 36.44 ± 2.07 & 0.11 ± 1.33 & 28.31 ± 7.40 \\
 &  & CLTV & \paddedcolorbox{LightCyan}{44.13} ± 3.47 & \paddedcolorbox{LightCyan}{10.37} ± 2.51 & \paddedcolorbox{LightCyan}{59.88} ± 7.91 \\ \cmidrule{3-6}
 & \multirow{3}{*}{IQL} & CLTV-TD & 3.94 ± 2.28 & 6.65 ± 2.60  & 46.52 ± 0.00 \\
 &  & CLTV-RS & 36.78 ± 1.33 & 8.91 ± 3.32 & 51.45 ± 3.03 \\
 &  & CLTV & \paddedcolorbox{LightCyan}{41.83} ± 0.63 & \paddedcolorbox{LightCyan}{16.28} ± 7.22 & \paddedcolorbox{LightCyan}{77.10} ± 5.16 \\ 
 \hline
\multirow{6}{*}{\rotatebox{90}{\Large Hopper}} & \multirow{3}{*}{CQL} & CLTV-TD & 7.30 ± 10.74 & 7.42 ± 6.90 & 30.42 ± 35.61 \\
 &  & CLTV-RS & 12.28 ± 11.99 & 6.79 ± 0.00 & 64.48 ± 35.66 \\
 &  & CLTV & \paddedcolorbox{LightCyan}{51.04} ± 3.21 & \paddedcolorbox{LightCyan}{56.23} ± 15.14 & \paddedcolorbox{LightCyan}{83.44} ± 16.95 \\ \cmidrule{3-6}
 & \multirow{3}{*}{IQL} & CLTV-TD & 0.19 ± 0.00 &  0.17 ± 0.00 & 42.31 ± 16.43 \\
 &  & CLTV-RS & 61.95 ± 6.24 & 30.79 ± 2.23 & 32.53 ± 12.08 \\
 &  & CLTV & \paddedcolorbox{LightCyan}{55.79} ± 4.44 & \paddedcolorbox{LightCyan}{39.56} ± 2.78 & \paddedcolorbox{LightCyan}{70.73} ± 4.11 \\ 
 \hline
\multirow{6}{*}{\rotatebox{90}{\Large Walker2d}} & \multirow{3}{*}{CQL} & CLTV-TD & 26.68 ± 20.97 & 74.62 ± 14.59 & 7.06 ± 4.56 \\
 &  & CLTV-RS & 41.53 ± 10.37 & 66.38 ± 40.72 & 2.59 ± 1.55 \\
 &  & CLTV & \paddedcolorbox{LightCyan}{70.45} ± 8.72 & \paddedcolorbox{LightCyan}{97.43} ± 7.74 & \paddedcolorbox{LightCyan}{30.23} ± 41.25 \\ \cmidrule{3-6}
 & \multirow{3}{*}{IQL} & CLTV-TD & 39.59 ± 13.09 & 18.84 ± 19.42 & 73.06 ± 5.16 \\
 &  & CLTV-RS & 64.97 ± 5.54 & 35.34 ± 17.71 & 89.05 ± 9.94 \\
 &  & CLTV & \paddedcolorbox{LightCyan}{68.37} ± 4.12 & \paddedcolorbox{LightCyan}{89.51} ± 9.03 & \paddedcolorbox{LightCyan}{110.74} ± 0.66 \\ 
\bottomrule
\end{tabular}
}
\end{table}
\subsection{Similarity-Reward Trade-off Parameters}
\label{par-impact}
We examine how the parameters \(\delta\) and \(\lambda\) affect performance across different domains, datasets, and algorithms in offline RL.

The parameter \(\delta\) balances the importance of transition similarity, which refers to the transition score or its relevance to the target domain, against the actual reward received, allowing the model to focus appropriately on both aspects during learning. Tuning \(\delta\) is important when applying the model to datasets with different dynamics, ensuring the model generalizes well without overfitting to the source data. 

The parameter \(\lambda\), on the other hand, plays a key role in balancing exploration and exploitation. A higher \(\lambda\) value encourages more exploration by allowing the model to take actions that may not immediately seem optimal, but could lead to better long-term outcomes. In contrast, a lower \(\lambda\) value favors exploitation, where the model sticks to actions that have previously yielded high rewards.

The heatmaps in \autoref{fig-deltalambda-cql} show that the choice of \(\delta\) and \(\lambda\) has a noticeable effect on CLTV (CQL) performance across different environments. Moderate values for both parameters generally lead to better results. For instance, in environments like Ant and Hopper, higher \(\lambda\) values enhance the balance between exploration and exploitation, while moderate \(\delta\) values allow for greater flexibility in learning without overfitting.

Similarly, the heatmaps in \autoref{fig-deltalambda-iql} show that performance in CLTV (IQL) is influenced by \(\delta\) and \(\lambda\), though the algorithm tends to be more stable across different settings. The heatmaps indicate that higher \(\lambda\) values help the model explore effectively, preventing it from getting stuck in suboptimal solutions. At the same time, moderate \(\delta\) values strike a balance between leveraging current estimates and improving both policy and value functions. This suggests that fine-tuning \(\lambda\) helps the model explore better, while \(\delta\) adjusts how much it sticks to what it has already learned.

When we analyze the impact of \(\lambda\) and \(\delta\) on CLTV (CQL) and CLTV (IQL) across different datasets, we see clear differences in how each algorithm reacts. In CLTV (CQL), especially in the Ant environment with the medium-expert dataset, performance is sensitive to changes in these parameters. For instance, as \(\lambda\) increases from 0.0 to 1.0, we see improvements in rewards. This highlights how \(\lambda\) helps manage the balance between policy conservativeness and exploration. In contrast, CLTV (IQL) appears to be less affected by changes in \(\lambda\) and \(\delta\). For example, in the Ant environment, CLTV (IQL) achieves consistently high rewards across various parameter settings. This shows that CLTV (IQL) handles exploration and exploitation more efficiently without needing significant tuning of these parameters. 

In summary, \(\lambda\) and \(\delta\) have different impacts on CLTV (CQL) and CLTV (IQL). For CLTV (CQL), higher \(\lambda\) values and moderate \(\delta\) values tend to result in better performance, especially in environments like Ant. On the other hand, CLTV (IQL) performs well across a wide range of \(\lambda\) and \(\delta\) values, reducing the need for fine-tuning. Understanding the effect of these parameters is essential for configuring algorithms in offline RL, where they can have a major influence on the performance of the learned policies.

\begin{figure*}[!ht]
\center
\includegraphics[width=\textwidth]{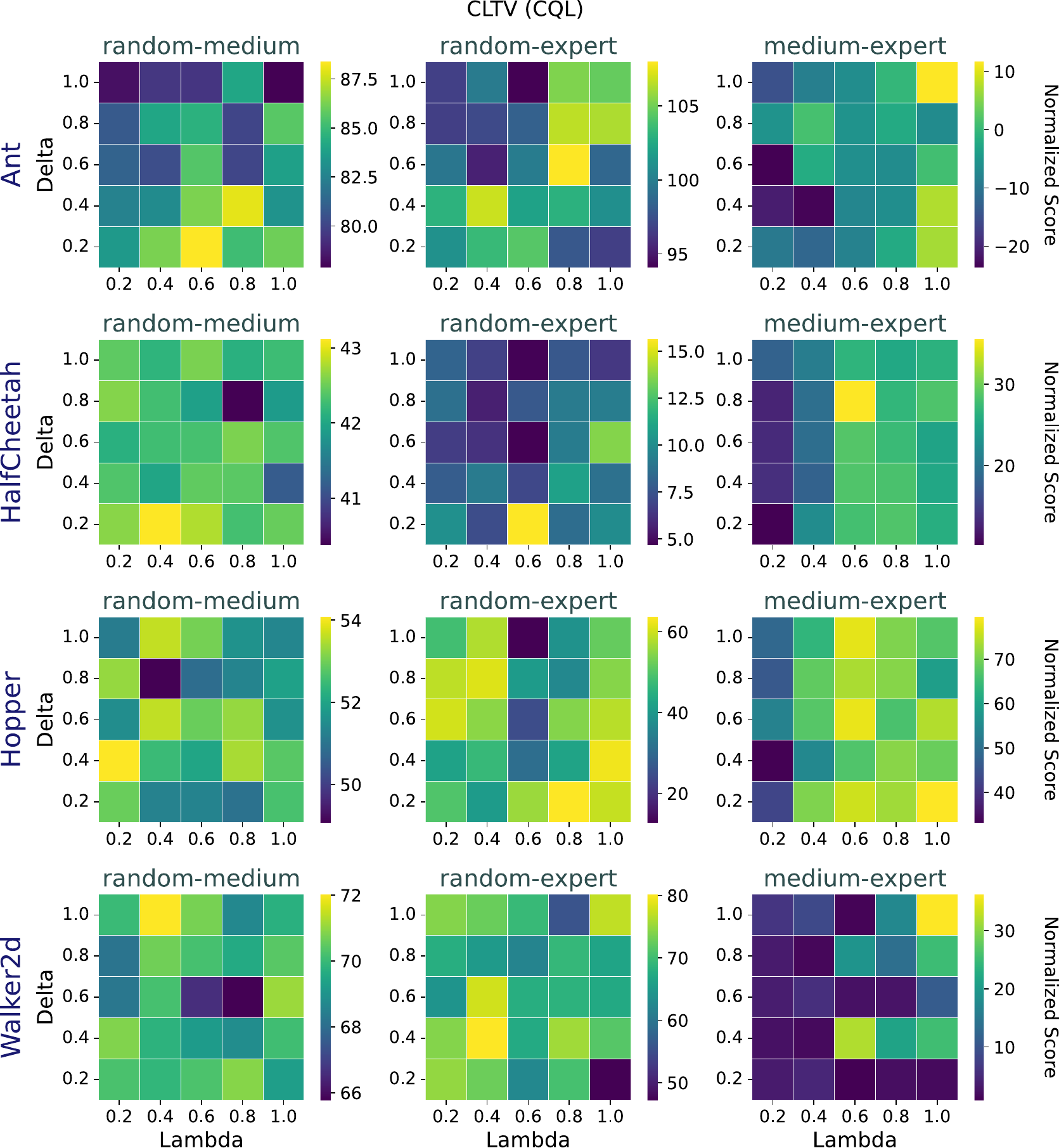}
\caption{{Heatmaps illustrating the performance of CLTV (CQL) on mixed datasets with respect to different $\delta$ (Delta) and $\lambda$ (Lambda) values, ranging from 0.2 to 1.0 in increments of 0.2, evaluated over 100 episodes with one seed.}}
\label{fig-deltalambda-cql}
\end{figure*}

\begin{figure*}[!ht]
\center
\includegraphics[width=\textwidth]{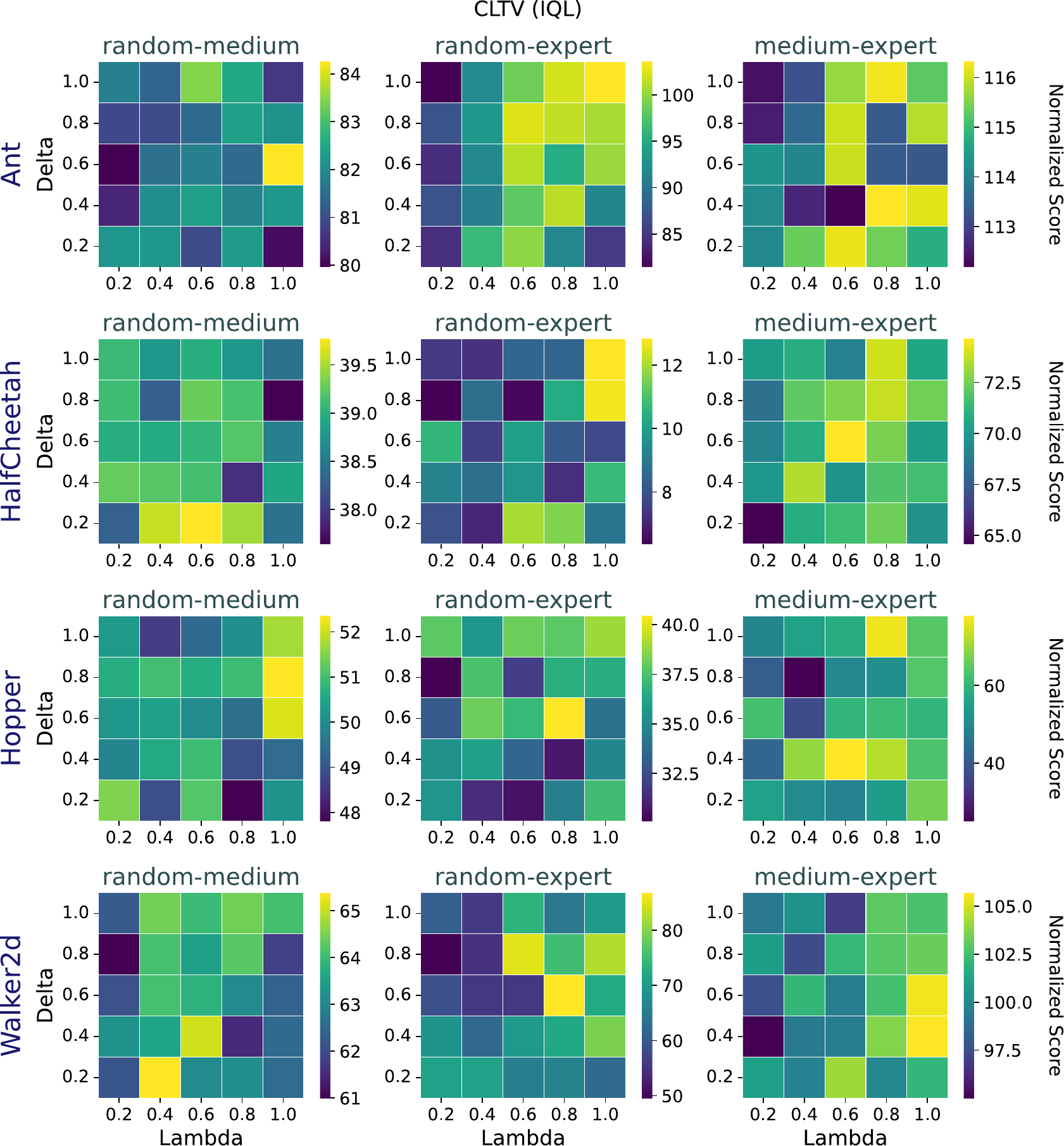}
\caption{{Heatmaps illustrating the performance of CLTV (IQL) on mixed datasets with respect to different $\delta$ (Delta) and $\lambda$ (Lambda) values, ranging from 0.2 to 1.0 in increments of 0.2, evaluated over 100 episodes with one seed.}}

\label{fig-deltalambda-iql}
\end{figure*}

\clearpage  
\makeatletter
\@twocolumnfalse
\makeatother

\begin{paracol}{2}  
    \switchcolumn[0]
    \section{Experimental Details}
\label{ex-details}

\subsection{Implementation and Computational Resources}
\label{sec-technical-details}
All the experiments were conducted on a high-performance GPU cluster consisting of five interconnected compute nodes, each equipped with eight NVIDIA A100 Tensor Core GPUs with 40 GB of memory per unit. Our method and the experiments were implemented in Python 3.10 under Ubuntu 22.10. Our code
is available at \href{https://github.com/amir-abolfazli/CLTV}{https://github.com/amir-abolfazli/CLTV}.

\subsection{Parameter Tuning}
\label{param-tuning}
For the base offline RL algorithms (CQL and IQL), we used the optimal hyperparameter values reported in~\citep{seno2022d3rlpy}, as listed in \autoref{tab:parameters-base}.

\begin{table}[!h]
\caption{Hyperparameter values of base offline RL methods.}
\label{tab:parameters-base}
\centering
\resizebox{0.95\linewidth}{!}{
\renewcommand{\arraystretch}{1.4}
\begin{tabular}{l l c l}
 \toprule
        & Hyperparameter  & Value & Description \\
    \hline
    CQL & Actor Learning Rate & $1 \times 10^{-4}$ & Learning rate for training policy \\
         & Critic Learning Rate & $3 \times 10^{-4}$ & Learning rate for training Q network \\
         & Temperature Learning Rate & $1 \times 10^{-4}$ & Learning rate for temperature parameter of SAC \\
         & Alpha Learning Rate $\tau$ & $1 \times 10^{-4}$ & The learning rate for parameter alpha \\
         & Discount Factor & $0.99$ & The factor of discounted return \\
         & Target Network $\tau$ & $5 \times 10^{-3}$ & The target network synchronization coefficiency \\
    \hline
    IQL & Actor Learning Rate & $3 \times 10^{-4}$ & Learning rate for training policy \\
        & Critic Learning Rate & $3 \times 10^{-4}$ & Learning rate for training Q network \\
        & Discount Factor & $0.99$ & The factor of discounted return \\
        & Target Network $\tau$ & $5 \times 10^{-3}$ & The target network synchronization coefficiency \\
        & Expectile & $0.7$ & The expectile value for value function training \\
 \bottomrule
\end{tabular}
}
\end{table}

For our CLTVORL method, we selected the optimal hyperparameter values by grid search, as listed in \autoref{tab:parametes-cltvorl}.

\begin{table}[!h]
\caption{Hyperparameter values of our CLTV method.}
\label{tab:parametes-cltvorl}
\centering
\resizebox{0.95\linewidth}{!}{
\renewcommand{\arraystretch}{1.4}
\begin{tabular}{l c l}
 \toprule
     Hyperparameter  & Value & Description \\
    \hline
    Score-Reward Ratio $ \lambda$ & $0.8$ & The ratio of transition score to transition reward of TS \\
    Similarity-Reward Ratio $\delta$ & $0.7$ & The ratio of transition similarity to the reward of TS \\
    Episode Ratio $m$ & 0.1 & The number of episodes sampled from the source dataset \\
    Batch Size & $200$ & The batch size of TS\\
    Hidden Layers & [256, 256] & The size of hidden layers for TS  \\
    Classifier Hidden Size & $256$ & The hidden size of classifiers \\
    Classifier Learning Rate & $3 \times 10^{-4}$ & Learning rate of classifiers \\
    Gaussian Standard Deviation & $0.1$ & The standard deviation of Gaussian distribution \\
 \bottomrule
\end{tabular}
}
\end{table}

    \switchcolumn
    \vfill
\end{paracol}

\end{document}